\documentclass[lettersize,journal]{IEEEtran}
\usepackage{amsmath,amsfonts,amsthm}
\usepackage{algorithmic}
\usepackage{algorithm}
\usepackage{array}
\usepackage[caption=false,font=normalsize,labelfont=sf,textfont=sf]{subfig}
\usepackage{textcomp}
\usepackage{stfloats}
\usepackage{url}
\usepackage{verbatim}
\usepackage{graphicx}
\usepackage{cite}
\usepackage{enumitem}

\newcommand{\R}{\mathbb{R}}

\newcommand{\N}{\mathbb{N}}

\newcommand{\Lc}{\mathcal{L}}

\newcommand{\Dc}{\mathcal{D}}

\newcommand{\ee}{\varepsilon}
\newcommand{\lo}{\longrightarrow}
\newcommand{\li}{\left}
\newcommand{\re}{\right}

\newcommand{\reg}{\mathrm{reg}}

\newtheorem{definition}{Definition}

\newtheorem{theorem}{Theorem}
\newtheorem{experiment}{Experiment}

\usepackage{todonotes}

\begin{document}

\title{Ensuring Topological Data-Structure Preservation under Autoencoder Compression due to Latent Space Regularization in Gauss--Legendre nodes}

\author{Chethan Krishnamurthy Ramanaik, Juan-Esteban Suarez Cardona, Anna Willmann, Pia Hanfeld, Nico Hoffmann and Michael Hecht\\
\IEEEmembership{Center for Advanced Systems Understanding (CASUS),Görlitz, Germany, \\ Helmholtz-Zentrum Dresden-Rossendorf e.V. (HZDR),Dresden Germany}
\thanks{This work was partially funded by the Center of Advanced Systems Understanding (CASUS), financed by Germany's Federal Ministry of Education and Research (BMBF) and by the Saxon Ministry for Science, Culture and Tourism (SMWK) with tax funds on the basis of the budget approved by the Saxon State Parliament.}
}



\maketitle

\begin{abstract}
  We formulate a data independent latent space regularisation constraint for general unsupervised autoencoders. The regularisation  rests on sampling the autoencoder Jacobian in Legendre nodes, being the centre of the Gauss--Legendre quadrature. Revisiting this classic enables to prove that regularised autoencoders ensure a one-to-one re-embedding of the initial data manifold to its latent representation.
  Demonstrations show that prior proposed regularisation strategies, such as contractive autoencoding,  cause topological defects already for simple examples, and so do convolutional based (variational) autoencoders. In contrast, topological preservation is ensured already by standard multilayer perceptron neural networks when being regularised due to our contribution.
  This observation extends
  through the classic FashionMNIST dataset up to real world encoding problems for MRI brain scans, suggesting that,
  across disciplines, reliable low dimensional representations of complex high-dimensional datasets can be delivered due to this regularisation technique.
\end{abstract}


\section{Introduction}
Systematic analysis and post-processing of high-dimensional and high-throughput datasets \cite{pepperkok2006high,perlman2004multidimensional}, is a current computational challenge across disciplines, such as neuroscience \cite{vogt2018machine,carlson2018ghosts,Zhang2020}, plasma physics \cite{karniadakis2021physics,rodriguez2019identifying, Willmann2021}, or cell biology and medicine \cite{kobayashi2022self, chandrasekaran2021image, anitei2014,nikitina2019}.
In the machine learning (ML) community,  autoencoders (AEs) are commonly considered as the central tool for learning a low-dimensional
\emph{one-to-one representation} of high-dimensional datasets. The representations  serve as a baseline for feature
selection and classification tasks, heavily arising for instance in bio-medicine \cite{ronneberger2015u,galimov2022tandem,yakimovich2020mimicry,fisch2020image,andriasyan2021microscopy}.

\begin{figure*}[t!]
\center
\includegraphics[scale=0.2]{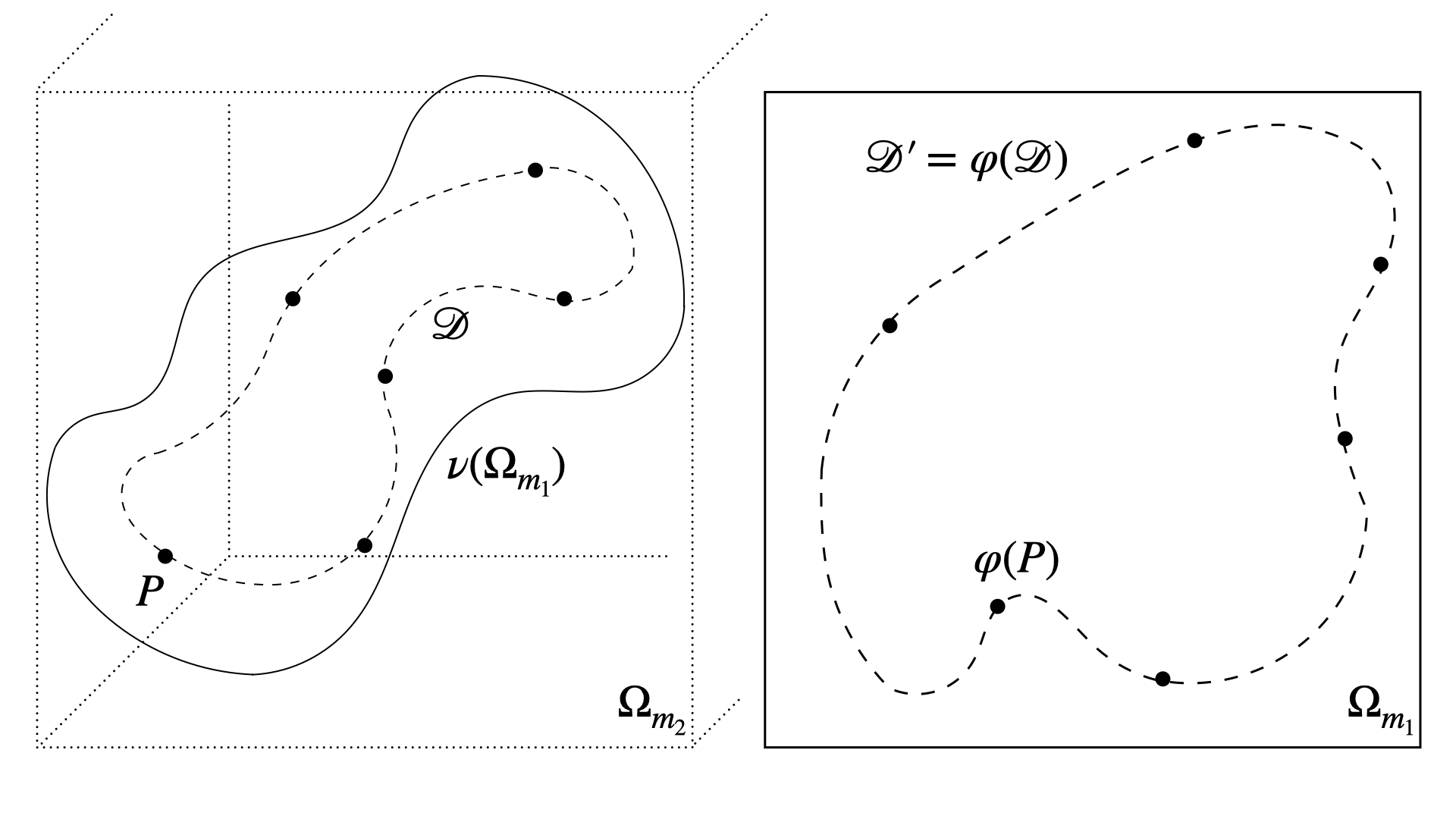}
\caption{Illustration of the latent representation $\Dc' = \varphi(\Dc) \subseteq \Omega_{m_1}$ of the data manifold $\Dc \subseteq \Omega_{m_2}$, $\dim \Dc = m_0 < m_1 < m_2 \in \N$ given by the autoencoder $(\varphi,\nu)$. The decoder is a one-to-one mapping of the hypercube $\Omega_{m_1}$ to its image $\nu(\Omega_{m_1}) \supset \Dc$, including $\Dc$ in its interior
and consequently guaranteeing Eq.~\eqref{eq:embed}.\label{Fig:DM}}
\end{figure*}

AEs can be considered as a \emph{non-linear extension} of classic \emph{principal component analysis} (PCA), \cite{sanchez1979,dunteman1989basic,krzanowski2000}, comparisons for linear problems are given in \cite{Rolinek2019}. While addressing the non-linear case, however, AEs are confronted with the challenge of guaranteeing topological data-structure preservation under AE compression.

\emph{To state the problem:} We mathematically formalise AEs as pairs of continuously differentiable  maps $(\varphi,\nu)$, $\varphi : \Omega_{m_2} \lo \Omega_{m_1}$,  $\nu : \Omega_{m_1} \lo \Omega_{m_2}$, $0<m_1 < m_2 \in \N$, defined on bounded domains $\Omega_{m_1}\subseteq \R^{m_1},\Omega_{m_2}\subseteq \R^{m_2}$.
Commonly, $\varphi$ is termed the \emph{encoder}, and $\nu$ the \emph{decoder}. We assume that the data $D\subseteq \Dc$ is sampled from a regular or even smooth data manifold $\Dc \subseteq \Omega_{m_2}$, $\dim \Dc =m_0 \leq m_1$.

We seek to find proper AEs $(\varphi,\nu)$ yielding homeomorphic latent representations
$\varphi(D)=\Dc' \cong \Dc$. In other words: the restrictions $\varphi_{|\Dc}: \Dc \lo \Dc'$, $\nu_{|\Dc'}: \Dc' \lo \Dc$ of the encoder and decoder result in one-to-one maps, being inverse to each other:
\begin{equation}\label{eq:embed}
\Dc \cong \Dc' := \varphi(\Dc) \subseteq \Omega_{m_1}\,,\quad \nu(\varphi(x))=x \,, \forall x \in \Dc\,.
\end{equation}
While the second condition is usually realised due to minimisation of reconstruction losses, this is insufficient for guaranteeing the one-to-one representation $\Dc \cong \Dc'$.

For realising proper AEs, matching both requirements in Eq.~\eqref{eq:embed},
we strengthen the condition by requiring the decoder to be an embedding of the whole latent domain $\Omega_{m_1}\supset \Dc$, including $\nu(\Omega_{m_1}) \supset \Dc$ in its interior,
see Fig.~\ref{Fig:DM} for an illustration. We mathematically prove and empirically demonstrate this \emph{latent regularisation strategy} to deliver regularised AEs (AR-REG),  satisfying Eq.~\eqref{eq:embed}.

Our investigations are motivated by recent results of  Hansen et al.,\cite{antun2021, gottschling2020,antun2020} complementing other contributions~\cite{NEURIPS2019_cd9508fd, galhotra2017fairness, mazzucato2021reduced} investigating instabilities of machine learning methods from a general mathematical perspective.

\subsection{The inherent instability of inverse problems}

The \emph{instability phenomenon of inverse problems} states that, in general, one
cannot guarantee solutions of inverse problems to be stable. An excellent introduction to the topic is given in \cite{antun2021} with deeper treatments and discussions  in  \cite{gottschling2020,antun2020}.

In our setup, these insights translate to the fact that, in general, the \emph{local Lipschitz constant}
\begin{equation*}
    L_\ee(\nu,y) = \sup_{0<\|y' - y\|<\ee} \frac{\|\nu(y') - \nu(y)\|}{\|y' -y\|}\,, \quad \ee >0
\end{equation*} of the decoder $\nu : \Omega_{m_1} \lo \Omega_{m_2}$ at some latent code $y \in \Omega_1$ might be unbounded. Consequently, small perturbations $y' \approx y$ of the latent code can result in large differences of the reconstruction $\|\nu(y') - \nu(y)\| \gg 0$.
This fact generally applies and can only be avoided if an additional control on the null space  of the Jacobian of the encoder $\ker J(\varphi(x))$ is given. Providing this control is the essence of our contribution.

\subsection{Contribution}
Avoiding the aforementioned instability, requires the null space of the encoder  Jacobian $\ker J(\varphi)$ to be perpendicular to the tangent space of $\Dc$
\begin{equation}\label{eq:STAB}
    \ker J(\varphi) \perp T\Dc \,.
\end{equation}
In fact, due to the \emph{inverse function theorem}, see e.g.\cite{lang1985,krantz2013}, the conditions Eq.~\eqref{eq:embed} and Eq.~\eqref{eq:STAB} are equivalent. In Fig.~\ref{Fig:DM}, $\ker J (\varphi)$ is illustrated to be perpendicular to the image of the whole latent domain $\Omega_{m_1}$
\begin{equation*}
    \ker J (\varphi) \perp T\nu(\Omega_{m_1})\,, \quad  \nu(\Omega_{m_1})\supseteq \Dc\,,
\end{equation*}
being  sufficient for guaranteeing  Eq.~\eqref{eq:STAB}, and consequently, Eq.~\eqref{eq:embed}.

While several state-of-the-art AE regularisation techniques are commonly established, none of them specifically formulates this necessary mathematical requirement in Eq.~\eqref{eq:STAB}.
Consequently, we are not aware of any regularisation approach that can theoretically guarantee the regularised AE to preserve the topological data-structure, as we do in Theorem~\ref{theo:AE}.
For realising a \emph{latent space regularised} AE (AE-REG) we introduce the $L^2$-regularisation loss
\begin{equation}\label{eq:loss}
    \Lc_{\reg}(\varphi,\nu) = \|J(\varphi\circ\nu) - I\|_{L^2(\Omega_{m_1})}^2
\end{equation}
and mathematically prove the AE-REG to satisfy condition Eq.~\eqref{eq:embed}, Theorem~\ref{theo:AE}, when being trained due to this additional regularisation.
To approximate $\Lc_{\reg}(\varphi,\nu)$ we revisit classic Gauss--Legendre quadratures (cubatures) \cite{stroud,stroud2,trefethen2017,Trefethen2019,sobolev1997theory}, only requiring sub-sampling of $J(\nu\circ\varphi)(p_\alpha)$, $p_\alpha \in P_{m,n}$
on a Legendre grid of sufficient high resolution $1\ll |P_{m,n}|$ in order to execute the regularization.
While the data independent latent Legendre nodes $P_{m,n}\subseteq \Omega_{m_1}$ are contained in the smaller dimensional latent space,
regularisation of high resolution can be efficiently realised.

Based on our prior work \cite{REG_arxiv, cardona2023learning,esteban2022replacing}, and \cite{PIP1,PIP2,MIP,hecht2018}, we complement the regularisation by a \emph{hybridisation approach} combining  autoencoders with
\emph{multivariate Chebyshev-polynomial-regression}. The resulting Hybrid AE is acting on the polynomial
coefficient space, given by pre-encoding the training data due to high-quality regression.

We want to emphasise that the proposed regularization is data independent in the sense that it does not require any prior knowledge of the data manifold, its embedding, or any parametrization of $\Dc$. Moreover, while being condensed into the loss, the regularization is independent of the AE architecture and can be applied to any AE realizations e.g., convolutional or variational AEs, whereas we show already regularized MLP based AEs to perform superior to the aforementioned ones.

As we demonstrate, the regulaization yields the desired re-emebdding, enhances the autoencoders reconstruction quality, and gains robustness under noise-perturbations.

\subsection{Related work - Regularisation of  autoencoders}

A multitude of supervised learning schemes, addressing representation learning tasks, are surveyed in \cite{sup1,sup2}. Self-supervised autoencoders rest on \emph{inductive bias learning} techniques \cite{mitchell1980need,gordon1995evaluation} in combination with vectorised autoencoders
\cite{wu2020vector,van2017neural}. However, the mathematical requirements, Eq.~\eqref{eq:embed}, Eq.~\eqref{eq:STAB} were not considered in these strategies at all. Consequently, one-to-one representations might only be realised due a well chosen inductive bias regularisation for rich datasets
\cite{kobayashi2022self}.

This article focus on regularisation techniques of purely unsupervised AEs. We want to mention the following  prominent approaches:
\begin{enumerate}[left=0pt,label=\textbf{R\arabic*)}]
  \item\label{r1} \emph{Contractive AEs} (ContraAE) \cite{rifai2011higher,rifai2011contractive}
 are based on an ambient Jacobian regularisation loss
 \begin{equation}\label{eq:CAE}
     \Lc_{\reg}^*(\varphi,\nu) = \|J(\nu\circ\varphi) - I\|_{L^2(\Omega_{m_2})}^2
 \end{equation}
 formulated in the ambient domain. This makes contraAEs arbitrarily contractive in perpendicular directions  $ (T\Dc)^\perp$ of $T\Dc$. However, this is in-sufficient to guarantee Eq.~\eqref{eq:embed}. In addition, the regularisation is data dependent, resting on the training dataset, and is computationally costly due to the large Jacobian $J \in \R^{m_2\times m_2}$, $m_2\gg m_1 \geq 1$.
 Several experiments in Section~\ref{sec:num} demonstrate contraAE to fail delivering topologically preserved representations.

 \item\label{r2} \emph{Variational AEs} (VAE), along with extensions like $\beta$-VAE, consist of \emph{stochastic encoders and decoders} and are commonly used for density estimation and generative modelling of complex distributions based on minimisation of the Evidence Lower Bound (ELBO) \cite{kingma2013auto,burgess2018understanding}. The variational latent space distribution induces an implicit regularisation, which is complemented by \cite{Kumar2020,Rhodes2021} due to a $l_1$-sparsity constraint of the decoder Jacobian.

 However, as the contraAE-constraint, this regularisation is computationally costly and insufficient for guaranteeing a one-to-one encoding, which is reflected in the degenerated representations appearing in Section~\ref{sec:num}.

\item\label{r3}  \emph{Convolutional AEs} (CNN-AE) are known to deliver one-to-one representations
 for a generic setup theoretically \cite{gilbert2017}. However, the implicit convolutions seems to prevent clear separation of tangent $T\Dc$ and perpendicular direction $(T\Dc)^\perp$ of the data manifold $\Dc$, resulting in topological defects already for simple examples, see Section~\ref{sec:num}.
\end{enumerate}

\section{Mathematical concepts}

We provide the mathematical notions on which our approach rests, starting by fixing the notation.

\subsection{Notation}
We consider neural networks (NNs) $\nu(\cdot,w)$ of fixed architecture $\Xi_{m_1,m_2}$, specifying number and depth of the hidden layers, the choice of  piece-wise smooth activation functions $\sigma(x)$, e.g.  ReLU or $\sin$, with input dimension $m_1$ and output dimension $m_2$. Further, $\Upsilon_{\Xi_{m_1,m_2}}$ denotes the parameter space of the weights and bias $w =(v,b) \in W=V\times B \subseteq \R^K$, $K \in \N$, see e.g. \cite{martin2009,goodfellow2016}.

We denote with $\Omega_m=(-1,1)^m$ the $m$-dimensional open \emph{standard hypercube}, with $\|\cdot\|$ the standard Euclidean norm on $\R^m$ and with $\|\cdot\|_p$, $1 \leq p\leq \infty$ the $l_p$-norm.
$\Pi_{m,n} = \mathrm{span}\{x^\alpha\}_{\|\alpha\|_\infty \leq n}$  denotes the $\R$-\emph{vector space of all real polynomials} in $m$ variables spanned by all monomials
$x^{\alpha} = \prod_{i=1}^mx_i^{\alpha_i}$ of \emph{maximum degree} $n \in \N$ and
$A_{m,n}=\{\alpha \in \N^m : \|\alpha\|_\infty  = \max_{i=1,\dots,m}\{|\alpha_i|\}\leq n\}$ the corresponding multi-index set.
For an excellent overview on functional analysis we recommend \cite{Adams2003,brezis2011,pedersen2018}. Here,
we consider the Hilbert space $L^2(\Omega_m,\R)$ of all \emph{Lebesgue measurable} functions $f:\Omega_m \lo \R$ with finite $L^2$-norm
$\|f\|_{L^2(\Omega_m)}^2  < \infty$ induced by the inner product
\begin{equation}\label{eq:L2}
    <f,g>_{L^2(\Omega_m)} = \int \limits_{\Omega_m}f\cdot g \, d\Omega_m\,, \,\, f,g \in L^2(\Omega,\R)\,.
\end{equation}
Moreover, $C^k(\Omega_m,\R)$, $k \in \N \cup\{\infty\}$ denotes the
\emph{Banach spaces}
of continuous functions being $k$-times continuously differentiable, equipped with the norm $\|f\|_{C^k(\Omega_m)} = \sum_{i=0}^k \sup_{x \in \Omega_m}|D^\alpha f(x)|$, $\|\alpha\|_1\leq k$.

\subsection{Orthogonal polynomials and Gauss-Legendre cubatures}
We follow \cite{gautschi2011,stroud,stroud2,trefethen2017,Trefethen2019} for recapturing:
Let $m,n\in \N$ and $P_{m,n} = \oplus_{i=1}^m \mathrm{Leg_n} \subseteq  \Omega_m$ be the $m$-dimensional Legendre grids, where $\mathrm{Leg_n}=\{p_0,\ldots,p_n\}$ are the $n+1$ \emph{Legendre nodes} given by the roots of the \emph{Legendre polynomials} of degree $n+2$.
We  denote $p_{\alpha} = (p_{\alpha_1}, \ldots, p_{\alpha_m}) \in P_{A_{m,n}}$, $\alpha \in {A_{m,n}}$. The Lagrange polynomials $L_{\alpha}\in \Pi_{A_{m,n}}$, defined by  $L_{\alpha}(p_\beta)=\delta_{\alpha,\beta}$, $\forall \, \alpha,\beta \in {A_{m,n}}$, where $\delta_{\cdot,\cdot}$
denotes the \emph{Kronecker delta}, are  given by
\begin{equation}\label{eq:Lag}
    L_{\alpha} = \prod_{i=1}^ml_{\alpha_i,i}\,, \quad l_{j,i} = \prod_{j \not = i, j=0}^m \frac{x_i -p_j}{p_i-p_j}\,.
\end{equation}
Indeed, the $L_\alpha$ are an orthogonal $L^2$-basis of $\Pi_{m,n}$,
\begin{equation}\label{eq:weight}
\li<L_{\alpha},L_{\beta}\re>_{L^2( \Omega_m)}=\int \limits_{ \Omega_m} L_{\alpha}(x)L_{\beta}(x)d\Omega_m = w_{\alpha} \delta_{\alpha,\beta}\,,
\end{equation}
where the \emph{Gauss-Legendre cubature weight} $w_{\alpha} = \|L_{\alpha}\|^2_{L^2(\Omega_m)}$ can be computed numerically.
Consequently, for any polynomial $Q \in \Pi_{m,2n+1}$ of degree $2n+1$ the following cubature rule applies:
\begin{equation}\label{eq:Gauss}
    \int \limits_{\Omega_m} Q(x)d\Omega_m = \sum_{\alpha \in {A_{m,n}}} w_{\alpha}Q(p_{\alpha})\,.
\end{equation}
Thanks to $|P_{m,n}| = (n+1)^m \ll (2n+1)^m$ this makes \emph{Gauss-Legendre integration} a very powerful scheme, yielding
\begin{equation}\label{eq:L2}
\li <Q_1,Q_2 \re>_{L^2(\Omega_m)}  = \sum_{\alpha \in {A_{m,n}}} Q_1(p_\alpha)Q_2(p_\alpha) w_\alpha \,,
\end{equation}
for all $Q_1,Q_2 \in \Pi_{m,n}$.

In light of this fact, we propose the following AE regularisation method.

\section{Legendre-latent-space regularisation for autoencoders}

The regularization is formulated from the perspective of classic differential geometry, see e.g. \cite{lang1985,chen1999,taubes2011,do2016}.
As introduced in Eq~\eqref{eq:embed}, we assume that the training data $D_{\mathrm{train}}\subseteq \Dc \subseteq \R^{m_2}$ is sampled from a regular data manifold. We formalise the notion of autoencoders:

\begin{definition}[autoencoders and data manifolds]\label{def:AE} Let $1 \leq m_0\leq m_1 \leq m_2\in \N$, $\Dc \subseteq \Omega_{m_2}$ be
a (data) manifold of dimension $\dim \Dc = m_0$. Given
continuously differentiable maps
$\varphi : \Omega_{m_2}\lo \Omega_{m_1}$,
$\nu :\Omega_{m_1} \lo \Omega_{m_2}$
such that:
\begin{enumerate}
\item[i)]  $\nu$ is a right-inverse of $\varphi$ on $\Dc$, i.e, $\nu(\varphi(x))=x$ for all $x \in \Dc$.
\item[ii)] $\varphi$ is a left-inverse of $\nu$, i.e, $\varphi(\nu(y)) = y$
for all $y \in \Omega_{m_1}$
\end{enumerate}
Then we call the pair
$(\varphi,\nu)$ a \emph{proper autoencoder}  with respect to $\Dc$.
\end{definition}

Given a proper AE $(\varphi,\nu)$, $\varphi$ yields a low dimensional homeomorphic re-embedding of
$\Dc \cong \Dc' = \varphi(\Dc) \subseteq \R^{m_1}$
as demanded in Eq.~\eqref{eq:embed} and illustrated in Fig.~\ref{Fig:DM}, fulfilling the stability requirement of Eq.~\eqref{eq:STAB}.

We formulate the following losses for deriving proper AEs:
\begin{definition}[regularisation loss]\label{def:REGL}
Let $\Dc \subseteq \Omega_{m_2}$ be
a $C^1$-data manifold of dimension $\dim \Dc = m_0 < m_1<m_2$ and $\emptyset \neq D_{\mathrm{train}}\subseteq \Dc$ be a finite training dataset.
For NNs
$\varphi(\cdot,u)\in \Xi_{m_2,m_1}$, $\nu(\cdot,w)\in \Xi_{m_1,m_2}$
with weights $(u,w) \in \Upsilon_{\Xi_{m_2,m_1}} \times \Upsilon_{\Xi_{m_1,m_2}}$, we define the loss
\begin{align*}
  \Lc_{D_{\mathrm{train}},n} : \Upsilon_{\Xi_{m_2,m_1}} \times \Upsilon_{\Xi_{m_1,m_2}} \lo \R^+\,, \\
  \Lc_{D_{\mathrm{train}},n}(u,w)  = \Lc_0(D_{\mathrm{train}},u,w) + \lambda \Lc_1(u,w,n)\,,
\end{align*}
where $\lambda>0$ is a hyper-parameter and
\begin{align}
\Lc_0(D_{\mathrm{train}},u,w)=\sum_{x \in D_{\mathrm{train}}}\|x -\nu(\varphi(x,u),w)\|^2  \label{Rec_Loss} \\
 \Lc_1(u,w,n)= \sum_{\alpha \in A_{m_1,n}} \|\mathop{I} - \mathop{J}\big(\varphi(\nu(p_\alpha,w),u)\big)\|^2\,,  \label{REG_Loss}
\end{align}
with $I\in \R^{m_1\times m_1}$ denoting the identity matrix, $p_{\alpha} \in P_{m_1,n}$ be the Legendre nodes, and  $J\big(\varphi(\nu(p_\alpha,w)\big)\in \R^{m_1\times m_1}$ the Jacobian.
\end{definition}

We show that the AEs with vanishing loss result to be proper AEs, Defintion~\ref{def:AE}.

\begin{theorem}[Main Theorem]\label{theo:AE} Let the assumptions of Definition~\ref{def:REGL} be satisfied,  and $\varphi(\cdot,u_n)\in \Xi_{m_2,m_1}$, $\nu(\cdot,w_n)\in \Xi_{m_1,m_2}$ be sequences of continuously differentiable NNs satisfying:
\begin{enumerate}
\item[i)] The loss converges $\Lc_{D_{\mathrm{train}},n} (u_n,w_n)  \xrightarrow[n\rightarrow \infty]{} 0$.
\item[ii)] The weight sequences converge
$$\lim_{n \rightarrow \infty }(u_n,w_n) =(u_\infty,w_\infty) \in \Upsilon_{\Xi_{m_2,m_1}} \times \Upsilon_{\Xi_{m_1,m_2}}\,.$$
    \item[iii)] The decoder satisfies $\nu(\Omega_{m_1},w_n)\supseteq \Dc$, $\forall n \geq n_0\in \N$ for some $n_0 \geq 1$.
\end{enumerate}
Then $(\varphi(\cdot,w_n),\nu(\cdot,u_n)) \xrightarrow[n \rightarrow \infty]{} (\varphi(\cdot,w_\infty),\nu(\cdot,u_\infty))$ uniformly converges to a proper autoencoder with respect to $\Dc$.
\end{theorem}
\begin{proof} The proof follows by combining several facts: First, the \emph{inverse function theorem} \cite{krantz2013} implies that any map $\rho \in C^1(\Omega_m,\Omega_m)$ satisfying \begin{equation}\label{eq:IFT}
    J(\rho(x)) = I\,, \,\, \forall \, x \in \Omega_m\,,  \,\,\text{and}\,\,  \rho(x_0)= x_0\,, \,\,
\end{equation}
for some $x_0\in \Omega_m$ is given by the identity, i.e., $\rho(x) =x$,
$\forall x \in \Omega_m$.

Secondly, the Stone-Weierstrass theorem \cite{weier1,weier2} states that any continuous map $\rho \in C^0(\Omega_m,\Omega_m)$, with coordinate functions  $\rho(x) = (\rho_1(x),\ldots,\rho_m(x))$ can be uniformly approximated by a polynomial map
$Q_\rho^n(x) =(Q_{\rho,1}^n(x),\ldots,Q_{\rho,m}^n(x))$, $Q_{\rho,i}^n(x) \in \Pi_{m,n}$, $1\leq i\leq m$, i.e, $\|\rho - Q_\rho^n\|_{C^0(\Omega_m)} \xrightarrow[n\rightarrow \infty]{} 0$.

Thirdly, while the NNs $\varphi(\cdot,w)$, $\nu(\cdot,u)$ depend continuously on the weights $u,w$, the convergence in $ii)$ is uniform.
Consequently, the convergence $\Lc_{D_{\mathrm{train}},n} (u_n,w_n)  \xrightarrow[n\rightarrow \infty]{} 0$ of the loss implies that any sequence of polynomial
approximations $Q_\rho^n(x)$ of the map
$\rho(\cdot) =\varphi(\nu(\cdot,w_{\infty}),u_{\infty})$ satisfies Eq.~\eqref{eq:IFT} in the limit for $n \rightarrow \infty$. Hence, $\varphi(\nu(y,w_{\infty}),u_{\infty}) = Q_\rho^\infty(y)=y$ for all $y \in \Omega_{m_1}$ yielding requirement $ii)$ of Definition~\ref{def:AE}.

Given that assumption $iii)$
is satisfied, in completion, requirement $i)$ of Definition~\ref{def:AE} holds, finishing the proof.
\end{proof}

Apart from ensuring topological maintenance, one seeks for high-quality reconstructions. We propose a novel hybridisation approach delivering both.

\section{Hybridisation of autoencoders due to polynomial regression}

The hybridisation approach rests on deriving \emph{Chebyshev Polynomial Surrogate Models} $ Q_{\Theta,d}$ fitting the initial training data $d \in D_{\mathrm{train}}\subseteq \Omega_{m_2}$. For the sake of simplicity, we motivate the setup in case of images:

Let $d = (d_{ij})_{1\leq i,j \leq r} \in \R^{r \times r}$ be the intensity values of an image on an equidistant pixel grid $G_{r\times r}=(g_{ij})_{1\leq i,j\leq r} \subseteq \Omega_2$ of resolution $r\times r$, $r\in \N$. We seek for a polynomial
\begin{equation*}
Q_\Theta : \Omega_{2} \lo\R\,, \quad  Q_\Theta \in \Pi_{2,n}\,,
    \end{equation*}
such that evaluating $Q_\Theta$,  $\Theta=(\theta_\alpha)_{\alpha\in A_{2,n}} \in \R^{|A_{2,n}|}$
on $G_{r\times r}$ approximates $d$, i.e., $Q_\Theta(g_{ij}) \approx d_{ij}$ for all $1\leq i,j\leq r$.
We model $Q_\Theta$ in terms of \emph{Chebyshev polynomials of first kind} well known to provide excellent approximation properties  \cite{Trefethen2019,REG_arxiv}:
\begin{equation}\label{PSM}
   Q_{\Theta}(x_1,x_2) = \sum\limits_{\alpha\in A_{2,n}}\theta_\alpha T_{\alpha_1}(x_1)T_{\alpha_2}(x_2)\,.
\end{equation}
The expansion is computed due to standard least-square fits:
\begin{equation}\label{eq:fit}
   \Theta_d = \mathrm{argmin}_{C \in \R^{|A_{2,n}|}}\|RC - d\|^2\,,
\end{equation}
where
$R=(T_{\alpha}(g_{ij}))_{1\leq i,j,\leq n, \alpha \in A_{2,n}} \in \R^{r^2 \times |A_{2,n}|}$, $T_\alpha = T_{\alpha_1}\cdot T_{\alpha_2}$
denotes the regression matrix.

Given that each image (training point) $d \in D_{\mathrm{train}}$ can be approximated with the same polynomial degree $n \in \N$,
we posterior train an autoencoder $\varphi,\nu)$, only acting on the polynomial coefficient space $\varphi : \R^{|A_{2,n}|} \lo \Omega_{m_1}$,
$\nu \Omega_{m_1} \lo \R^{|A_{2,n}|}$   by exchanging
the loss in Eq.~\eqref{Rec_Loss} due to
\begin{equation}\label{Rec_Loss_hyb}
  \Lc_0^*(D_{\mathrm{train}},u,w)=\sum_{d \in D_{\mathrm{train}}}\|d -R\cdot \nu(\varphi(\Theta_d,u),w)\|^2
\end{equation}

In contrast to the regularisation loss in Definition~\ref{def:REGL}, here,  pre-encoding the training data due to polynomial regression decreases the input dimension $m_2 \in \N$ of the (NN) encoder $\varphi : \Omega_{m_2} \lo \Omega_{m_1}$.
In practice, this enables to reach low dimensional latent dimension by increasing the reconstruction quality, as we demonstrate in the next section.

\section{Numerical Experiments}\label{sec:num}
We executed experiments, designed to validate our theoretical results, on  {\sc hemera} a NVIDIA V100 cluster at HZDR. Complete code benchmark sets and supplements are available at  https://github.com/casus/autoencoder-regularisation.
The following AEs were applied:

\begin{enumerate}[left=0pt,label=\textbf{B\arabic*)}]
    \item\label{b1} \emph{Multilayer perceptron autoencoder  (MLP-AE)}: Feed forward NNs with activation functions $\sigma(x) = \sin(x)$.

    \item \emph{Convolutional autoencoder (CNN-AE)}: Standard convolutional neural networks (CNNs) with activation functions $\sigma(x) = \sin(x)$, as discussed in \ref{r3}.

    \item \emph{Variational autoencoder}: MLP based (MLP-VAE) and CNN based (CNN-VAE) as in \cite{kingma2013auto,burgess2018understanding}, discussed in \ref{r2}.
    \item  \emph{Contractive autoencoder (ContraAE)}: MLP based with with activation functions $\sigma(x) = \sin(x)$ as in \cite{rifai2011higher,rifai2011contractive}, discussed in \ref{r1}.
    \item  \emph{Regularised autoencoder (AE-REG)}:
    MLP based, as in \ref{b1}, trained with respect to the regularisation loss from Definition~\ref{def:REGL}.
    \item \emph{Hybridised AE (Hybrid AE-REG)}: MLP based, as in \ref{b1}, trained with respect to the modified loss in Definition~\ref{def:REGL} due to Eq.~\eqref{Rec_Loss_hyb}.
\end{enumerate}

The choice of activation functions $\sigma(x) = \sin(x)$
yields a natural way for normalising the latent encoding to $\Omega_m$ and performed best compared to trials with ReLU, ELU or $\sigma(x)=\tanh(x)$.
The regularisation of AE-REG and Hybrid AE-REG is realised due to sub-sampling batches from the Legendre grid $P_{m,n}$ for each iteration and computing the autoencoder Jacobians due to automatic differentiation \cite{baydin2018}.

\begin{figure}
         \centering
         \includegraphics[width=0.2\textwidth]{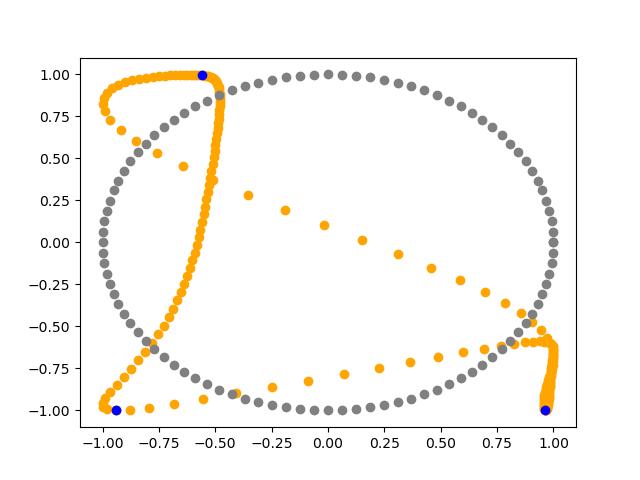}
         \includegraphics[width=0.2\textwidth]{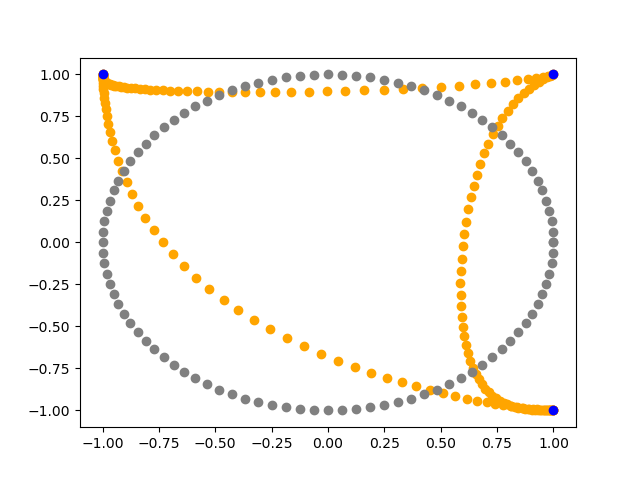}
         \caption{circle reconstruction: CNN-AE (left), CNN-VAE (right)}
         \label{circ1}
     \end{figure}
\begin{figure}
         \centering
         \includegraphics[width=0.2\textwidth]{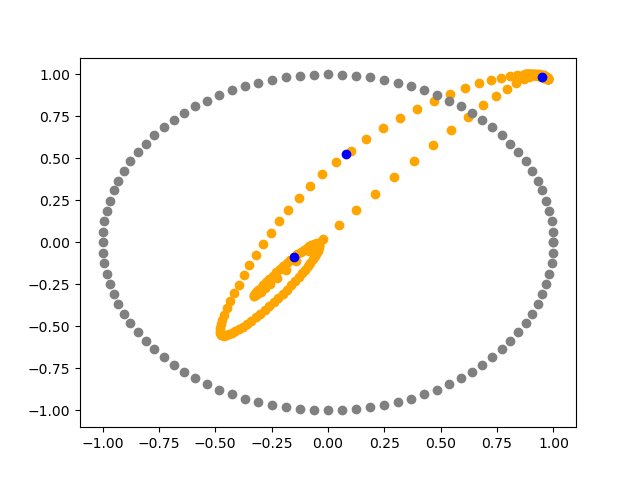}
         \includegraphics[width=0.2\textwidth]{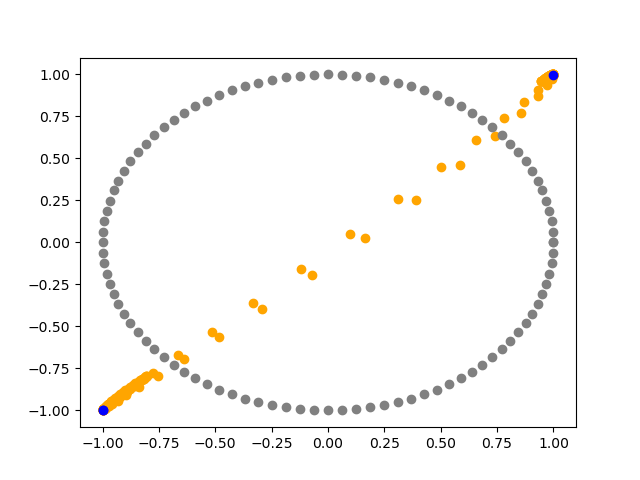}
         \caption{circle reconstruction:  ContraAE (left), MLP-VAE (right)}
         \label{circ2}
     \end{figure}
     \begin{figure}
          \centering
          \includegraphics[width=0.2\textwidth]{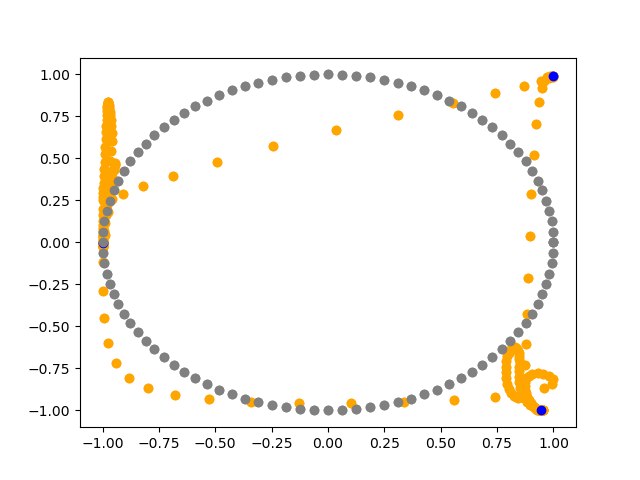}
         \includegraphics[width=0.2\textwidth]{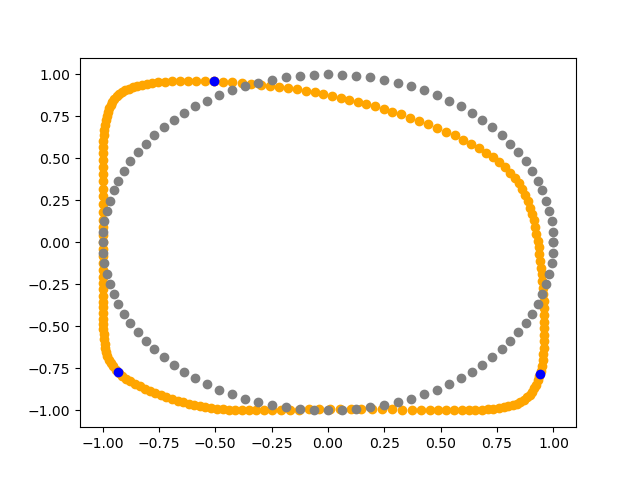}
         \caption{circle reconstruction: MLP-AE (left), AE-REG (right).}
          \label{circ3}
     \end{figure}

\subsection{Topological data-structure preservation}

Inspired by Fig.~\ref{Fig:DM}, we start by validating Theorem~\ref{theo:AE} for known data manifold topologies.

\begin{experiment}[Cycle reconstructions in dimension 15]  We consider the unit circle $S^1 \subseteq \R^2$, a uniform random matrix $A \in \R^{15,2}$ with entries in $[-2,2]$ and the data manifold $\Dc = \{Ax : x \in S^1\subseteq \R^2\}$, being an ellipse embedded along some $2$-dimensional hyperplane $H_A=\{Ax : x\in \R^2\} \subseteq \R^{15}$. Due to Bezout's Theorem
\cite{bezout1779,fulton1974}, a $3$-points sample uniquely determines an ellipse in a $2$-dimensional plane. Therefore, we  executed the AEs for this minimal case of a set of random samples  $|D_{\mathrm{train}}|=3$, $D_{\mathrm{train}} \subseteq \Dc$ as training set.
\label{exp:cyc}
\end{experiment}

MLP-AE, MLP-VAE, and AE-REG consists of 2 hidden linear layers (in the encoder and decoder), each of length 6. The degree of the Legendre grid $P_{m,n}$ used for the regularisation of AE-REG was set to $n=21$, Definition~\ref{def:REGL}.  CNN-AE and CNN-VAE consists of 2 hidden convolutional layers with kernel size 3, stride of 2 in the first hidden layer and 1 in the second, and 5 filters per layer. The resulting parameter spaces $\Upsilon_{\Xi_{15,2}}$ are all of similar size: $|\Upsilon_{\Xi_{15,2}}| \sim  400$.
All AEs were trained with the Adam optimizer \cite{kingma2014}.

 Representative results out of 6 repetitions are shown in Fig.~\ref{circ1}-\ref{circ3}. Only AE-REG delivers a feasible 2D re-embedding, while all other AEs cause overlappings or cycle-crossings. More examples are given in the supplements; whereas AE-REG delivers similar reconstructions for all other trials while the other AEs fail in most of the cases.

Linking back to our initial discussions of ContraAE \ref{r1}: The results show that the ambient domain regularisation formulated for the ContraAE,
is insufficient for guaranteeing a one-to-one encoding.
Similarily, CNN-based AEs cause self-intersecting points. As initially discussed in \ref{r3}, CNNs are invertible for a generic setup \cite{gilbert2017}, but seem to fail sharply separating tangent $T\Dc$ and perpendicular direction $T\Dc^\perp$ of the data manifold $\Dc$.

We demonstrate the impact of the regularisation to not belonging to an edge case by considering the following scenario:

\begin{experiment}[Torus reconstruction] Following the experimental design of Experiment~\ref{exp:cyc} we generate challenging tori embeddings of a squeezed torus with radii $ 0<r,R$, $r=0.7$ $R= 2.0$ in dimension $m=15$ and dimension $m=1024$ due to multiplication with random matrices
$A \in [-1,1]^{m\times 3}$. We randomly sample 50 training points and seek for their 3D re-embedding due to the AEs. For the low dimensional case the Legendre regularisation grid $P_{m,n}$ of degree $n=21$ is chosen, while for the high-dimensional task, $m=1024$, we choose $n=51$.
\label{exp:tor}
\end{experiment}

Show-cases for dimension $m=15$ are given in Fig.~\ref{torus1}-\ref{torus3}, visualised by a dense set of $2000$ test points.
As in Experiment~\ref{exp:cyc} only AE-REG is capable for re-embedding the torus in a feasible way.

CNN-AE and CNN-VAE flatten the torus, again non-preserving the torus topology. Apart from AE-REG all other AEs fail to provide a re-embedding.
For the high-dimensional task $m=1024$,
CNN-AE and CNN-VAE flatten, stretch or squeeze the torus, loosing its "third dimension". MLP-VAE, and MLP-AE show self-intersections, while ContraAE completely collapses.

Summarising the results suggests that only AE-REG is capable of preserving the data topology.
We continue our evaluation to give further evidence on this expectation.

\begin{figure}[t!]
    \centering
         \includegraphics[width=0.2\textwidth]{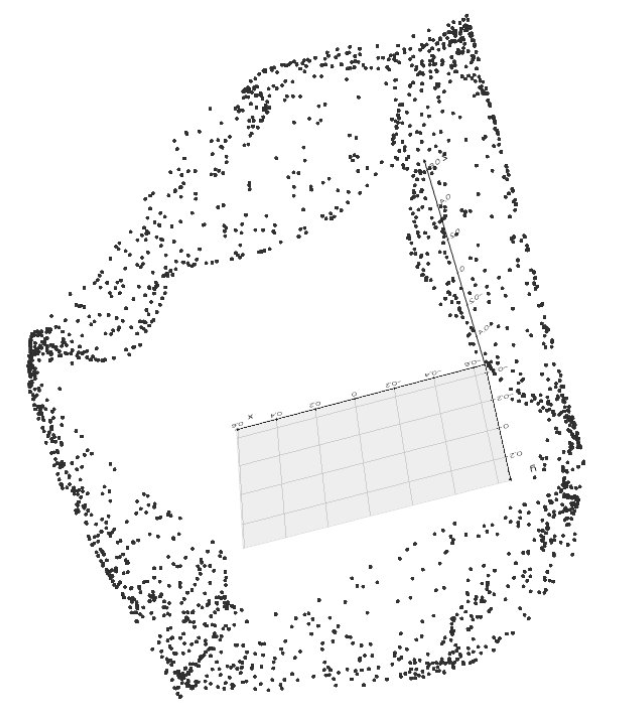}
         \includegraphics[width=0.2\textwidth]{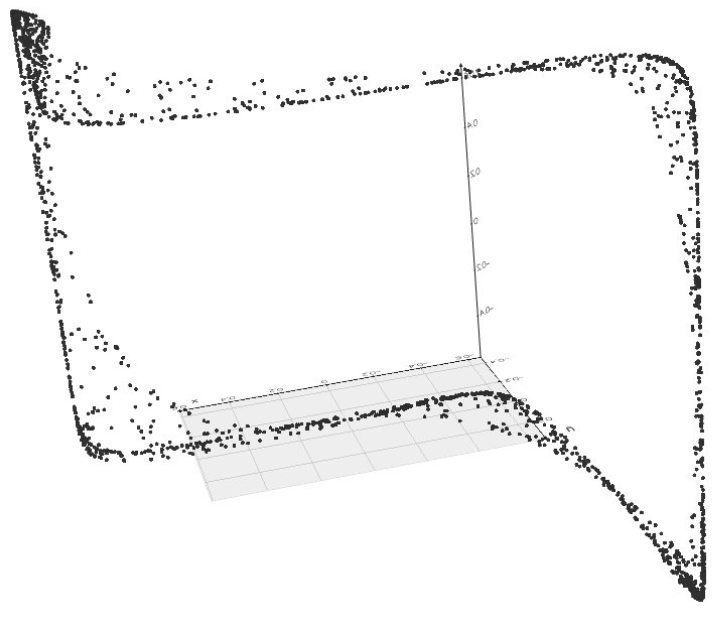}
         \caption{torus reconstruction, $\dim =15$: CNN-AE (left), CNN-VAE (right)}
         \label{torus1}
\end{figure}
\begin{figure}[t!]
    \centering
         \includegraphics[width=0.2\textwidth]{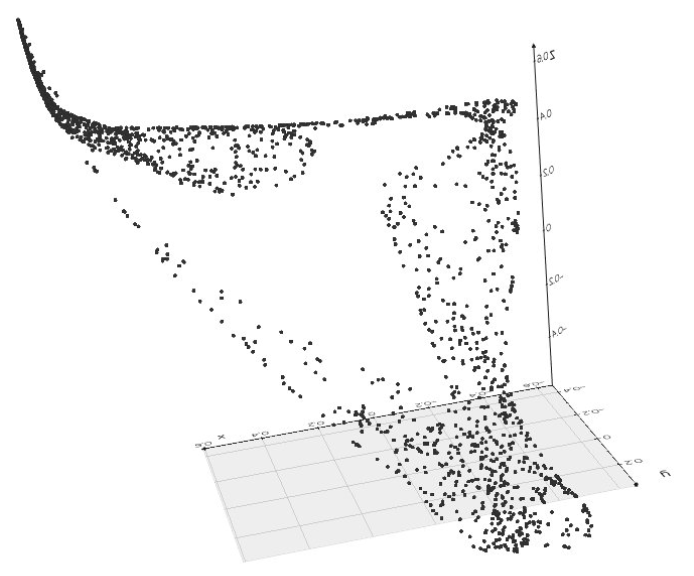}
         \includegraphics[width=0.2\textwidth]{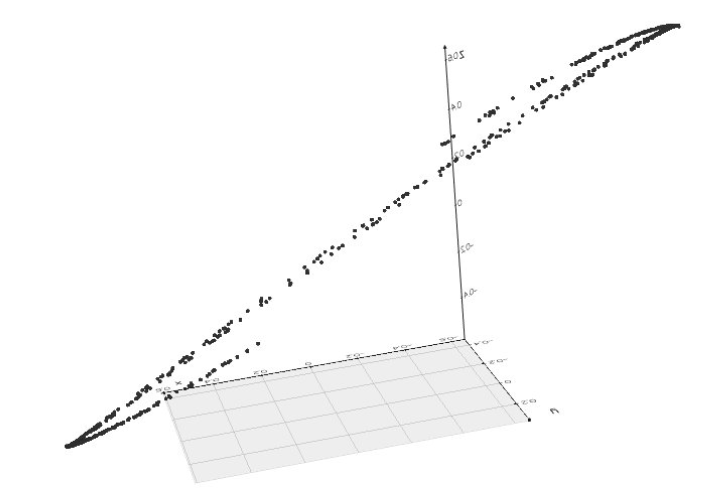}
         \caption{torus reconstruction, $\dim =15$: ContraAE(left), MLP-VAE (right)}
         \label{torus2}
\end{figure}
\begin{figure}[t!]
    \centering
         \includegraphics[width=0.2\textwidth]{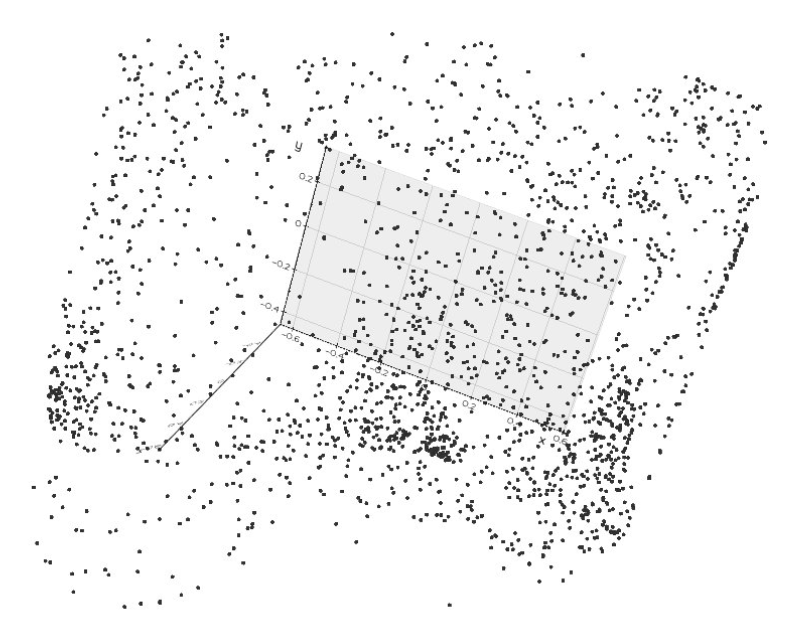}
         \includegraphics[width=0.2\textwidth]{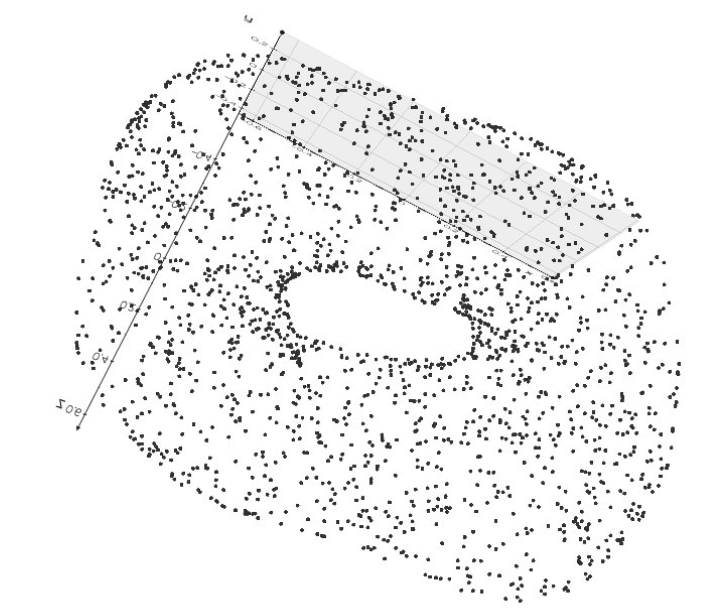}
         \caption{torus reconstruction, $\dim =15$: MLP-AE (left), AE-REG (right)}
         \label{torus3}
\end{figure}

\subsection{Autoencoder compression for FashionMNIST}
We continue by investigating the AE performances for the classic FashionMNIST dataset  \cite{xiao2017/online}.

\begin{experiment}[FashionMNIST compression]\label{exp:FM}
The 70 000 FashionMNIST images separated into 10 fashion classes (T-shirts, shoes etc.) being of $32\times32 =1024$-pixel resolution (ambient domain dimension). For providing a challenging competition we reduced the dataset to $24000$ uniformly sub-sampled images and trained the AEs for $40\%$ training data and complementary test data, respectively. Here, we consider latent dimensions $m=4,10$. Results of further runs for  $m=2,4,6,8,10$ are given in the supplements.
\end{experiment}

\begin{figure}[t!]
    \centering
         \includegraphics[width=0.2\textwidth]{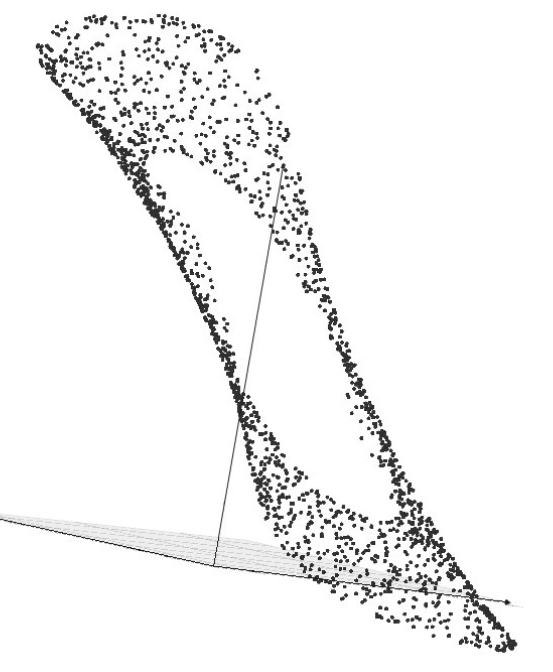}
         \includegraphics[width=0.2\textwidth]{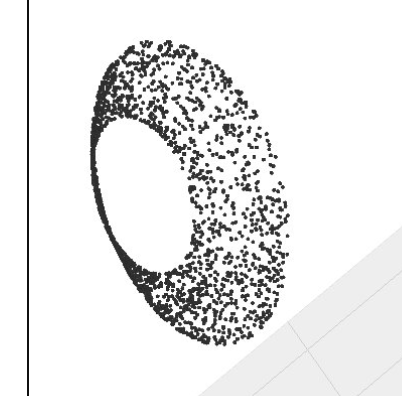}
         \caption{torus reconstruction, $\dim =1024$: CNN-AE (left), CNN-VAE (right)}
         \label{torus4}
\end{figure}
\begin{figure}[t!]
    \centering
         \includegraphics[width=0.2\textwidth]{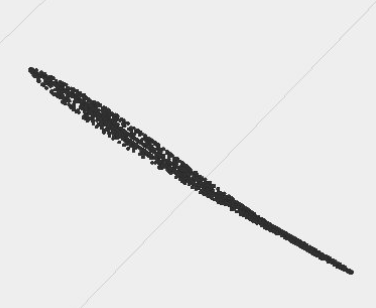}
         \includegraphics[width=0.2\textwidth]{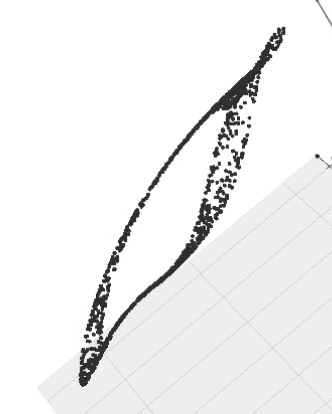}
         \caption{torus reconstruction, $\dim =1024$: ContraAE(left), MLP-VAE (right)}
         \label{torus5}
\end{figure}
\begin{figure}[t!]
    \centering
         \includegraphics[width=0.2\textwidth]{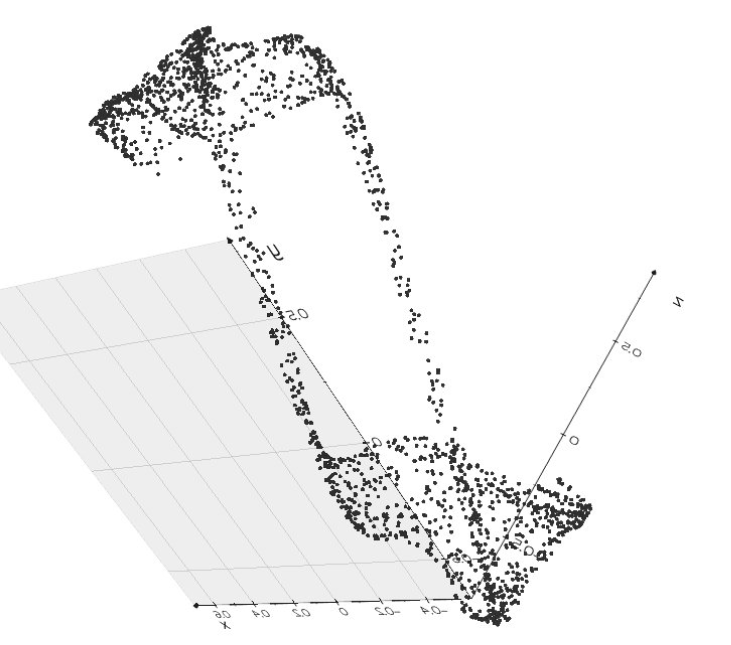}
         \includegraphics[width=0.2\textwidth]{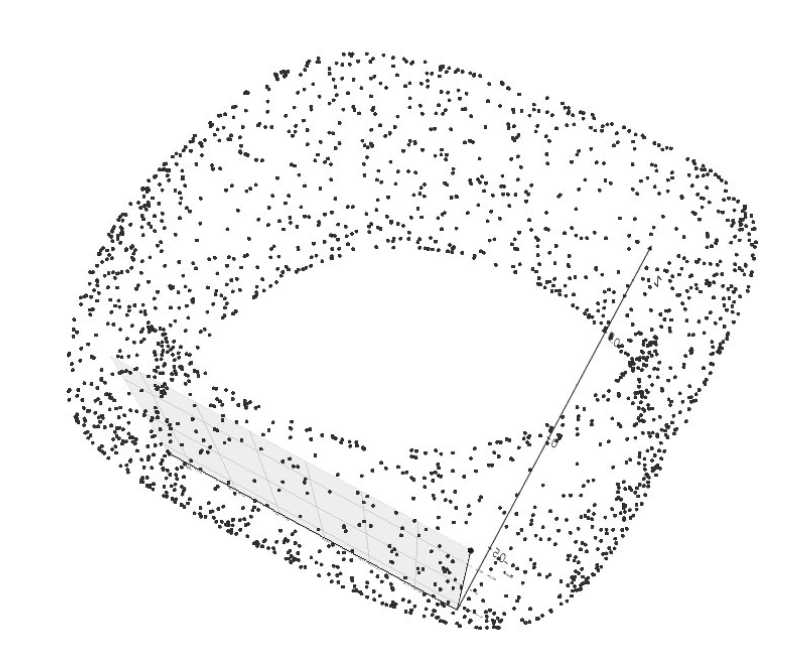}
         \caption{torus reconstruction, $\dim =1024$: MLP-AE (left), AE-REG (right)}
         \label{torus6}
\end{figure}

MLP-AE, MLP-VAE, AE-REG and Hybrid AE-REG consists of 3 hidden layers, each of length 100. The degree of the Legendre grid $P_{m,n}$ used for the regularisation of AE-REG was set to $n=21$, Defintion~\ref{def:REGL}.  CNN-AE and CNN-VAE consists of 3 convolutional layers with kernel size 3, stride of 2. The resulting parameter spaces $\Upsilon_{\Xi_{15,2}}$ of all AEs are of similar size. Further details of the architectures are reported in the supplements.

We evaluated the reconstruction quality with respect to peak-signal-to-noise-ratios (PSNR) for perturbed test data due   $0\%,10\%,20\%,50\% $ of Gaussian noise encoded to latent dimension $m=10$, and plot them in Fig.~\ref{FM1} and Fig.~\ref{FM2}.  The choice $m=10$, here, reflects the number of FashionMNIST-classes.

\begin{figure}
         \centering
         \includegraphics[width=0.225\textwidth]{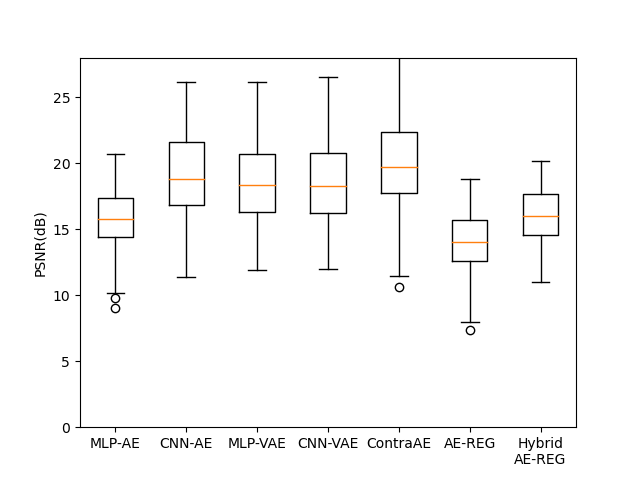}
         \includegraphics[width=0.225\textwidth]{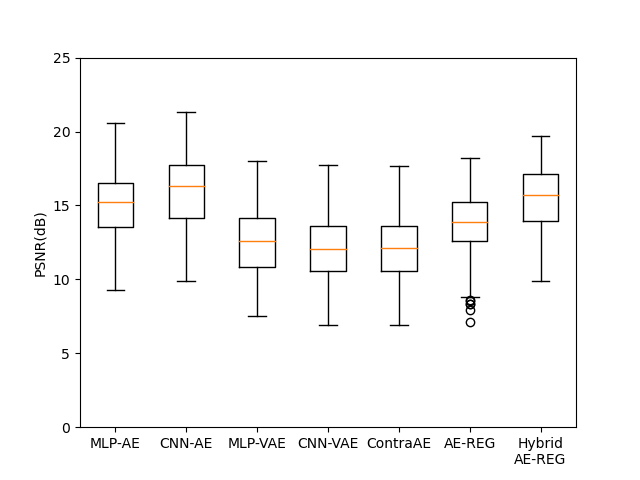}
         \caption{FashionMNIST reconstruction, latent dimension $\dim =10$: without noise(left) $10\%$ Gaussian noise (right)}
         \label{FM1}
\end{figure}
\begin{figure}[t!]
         \centering
         \includegraphics[width=0.225\textwidth]{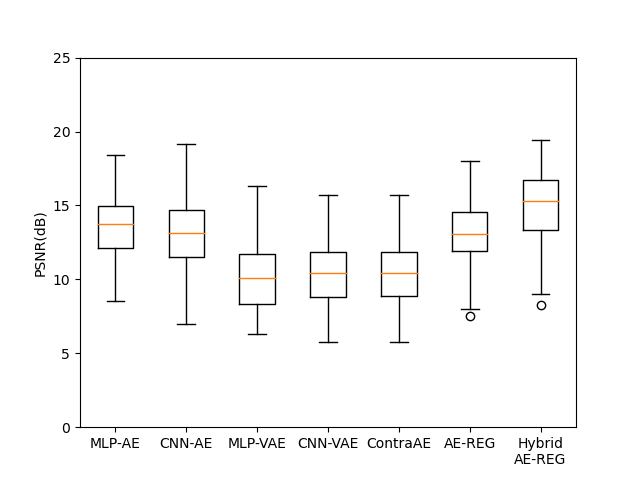}
         \includegraphics[width=0.225\textwidth]{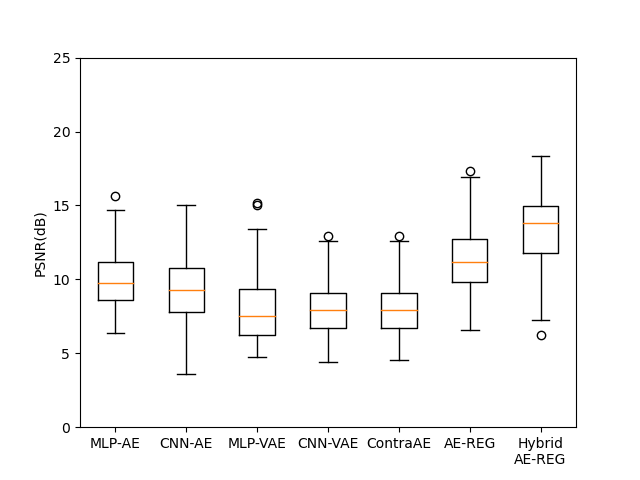}
        \caption{FashionMNIST reconstruction, latent dimension $\dim =10$: $20\%$ Gaussian noise (left) $50\%$ Gaussian noise (right)}
        \label{FM2}
\end{figure}

\begin{figure}[!t]
  \centering
 \includegraphics[width=0.5\textwidth]{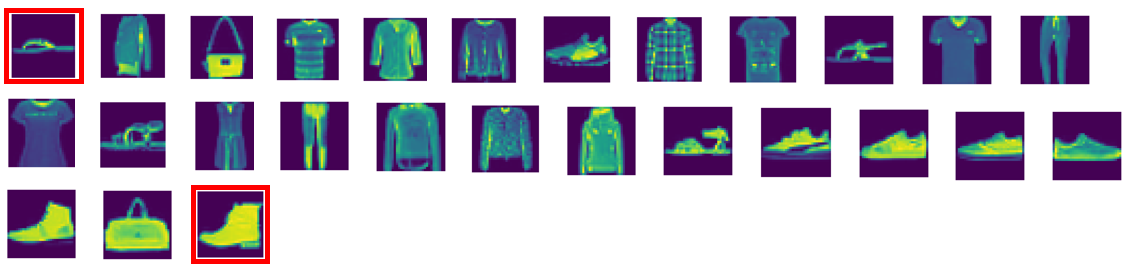}
 \caption{FashionMNIST goedesic in latent dimension $\dim =4$ of MLP-AE}
 \label{FG1}
\end{figure}
\begin{figure}[!t]
 \centering
\includegraphics[width=0.5\textwidth]{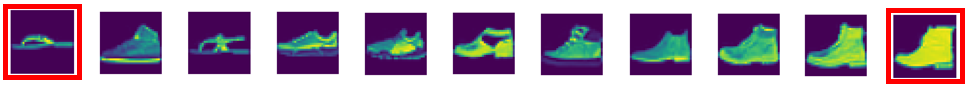}
         \caption{FashionMNIST goedesic in latent dimension $\dim =4$ of CNN-AE}
\end{figure}
\begin{figure}[!t]
  \centering
  \includegraphics[width=0.5\textwidth]{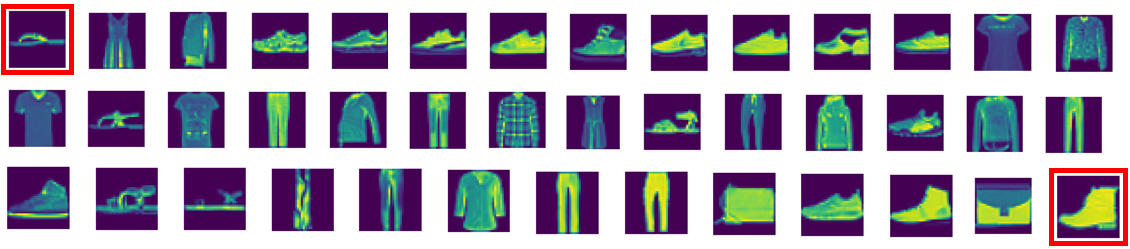}
 \caption{FashionMNIST goedesic in latent dimension $\dim =4$ of MLP-VAE}
       \end{figure}
\begin{figure}[!t]
  \centering
  \includegraphics[width=0.5\textwidth]{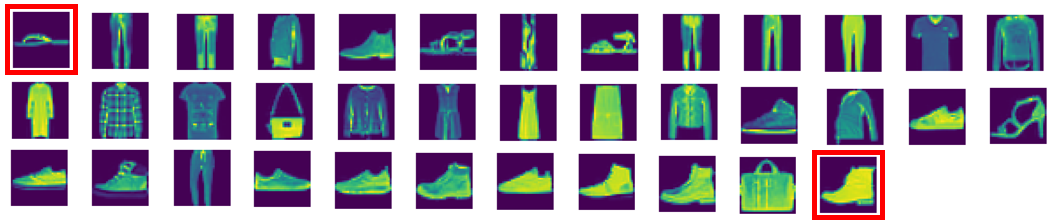}
 \caption{FashionMNIST goedesic in latent dimension $\dim =4$ of CNN-VAE}
       \end{figure}
\begin{figure}[t!]
  \centering
  \includegraphics[width=0.5\textwidth]{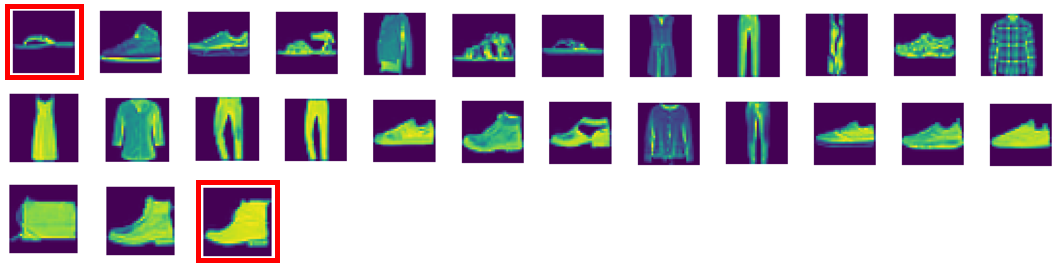}
\caption{FashionMNIST goedesic in latent dimension $\dim =4$ of ContraAE}
       \end{figure}
\begin{figure}[!t]
  \centering
  \includegraphics[width=0.5\textwidth]{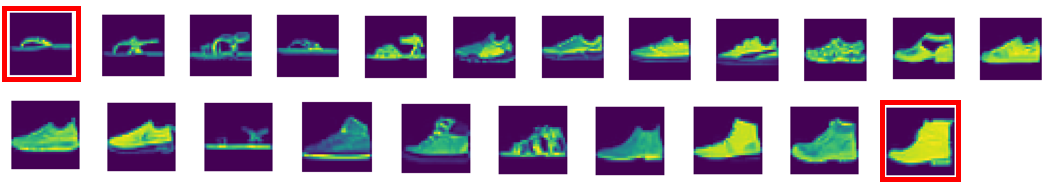}
  \caption{FashionMNIST goedesic in latent dimension $\dim =4$ of AE-REG}
       \end{figure}
\begin{figure}[!ht]
  \centering
 \includegraphics[width=0.5\textwidth]{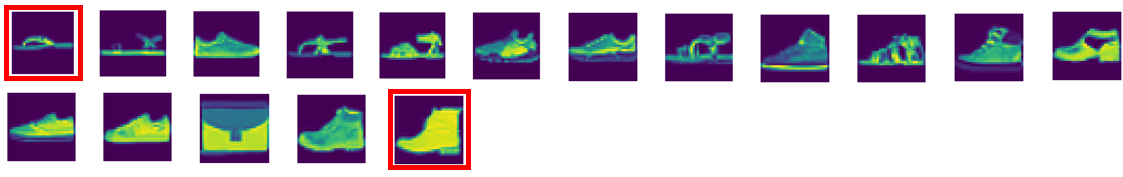}
 \caption{FashionMNIST goedesic in latent dimension $\dim =4$ of Hybrid AE-REG}
 \label{FGN}
\end{figure}

While Hybrid AE-REG performs compatible to MLP-AE and worse than the other AEs in the non-perturbed case, its superiority appears already for perturbations with $10\%$ of Gaussian noise and exceeds the reached reconstruction quality of all other AEs
for $20\%$ Gaussian noise or more. We want to stress that Hybrid AE-REG maintains its reconstruction quality throughout the noise perturbations (up to $70\%$, see the supplements). This outstanding appearance of robustness gives strong evidence on the impact of the regularisation and well-designed pre-encoding technique due to the hybridisation with polynomial regression.
Analogue results appear when measuring the reconstruction quality with respect to the structural similarity index measure (SSIM), given in the supplements.

In Fig.~\ref{fig:FM_showcase}  show cases of the reconstructions are illustrated, including in addition vertical and horizontal flip perturbations. Apart from AE-REG and Hybrid~AE-REG (rows (7) and (8)) all other AEs flip the FashionMNIST label-class for reconstructions of images with $20\%$ or $50\%$ of Gaussian noise. Flipping the label-class is the analogue to topological defects as cycle crossings appeared for the non-regularised AEs in Experiment~\ref{exp:cyc}, indicating again that the latent representation of the FashionMNIST dataset given due to the non-regularised AEs does not maintain structural information.

\begin{figure}[!t]
         \centering
         \includegraphics[width=0.225\textwidth]{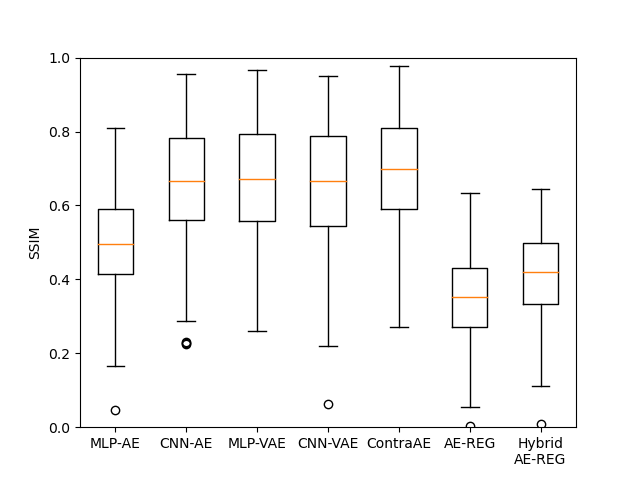}
         \includegraphics[width=0.225\textwidth]{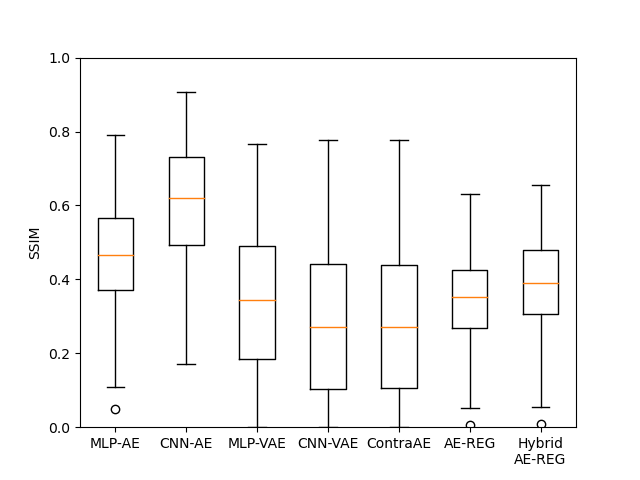}
         \caption{FashionMNIST reconstruction, latent dimension $\dim= 4$: without noise(left), $10\%$ Gaussian noise (right)}
         \label{FM3}
\end{figure}
\begin{figure}[!t]
         \centering
         \includegraphics[width=0.225\textwidth]{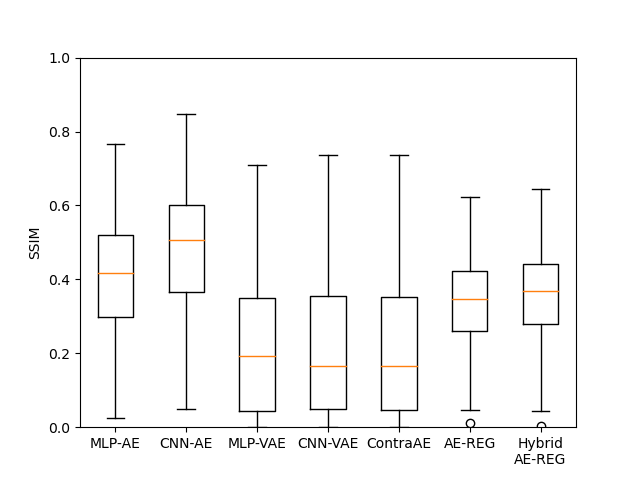}
         \includegraphics[width=0.225\textwidth]{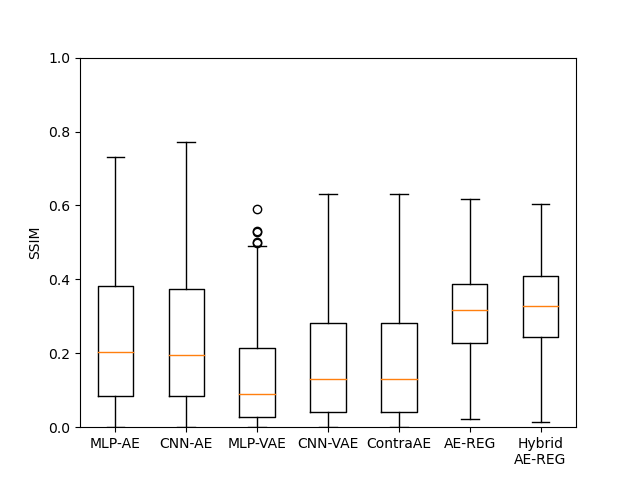}
        \caption{FashionMNIST reconstruction latent dimension $\dim =4$: $20\%$ Gaussian noise (left), $50\%$ Gaussian noise (right)}
        \label{FM4}
\end{figure}

\begin{figure*}[!ht]
     \vspace{-12pt}
         \includegraphics[width=0.49\textwidth]{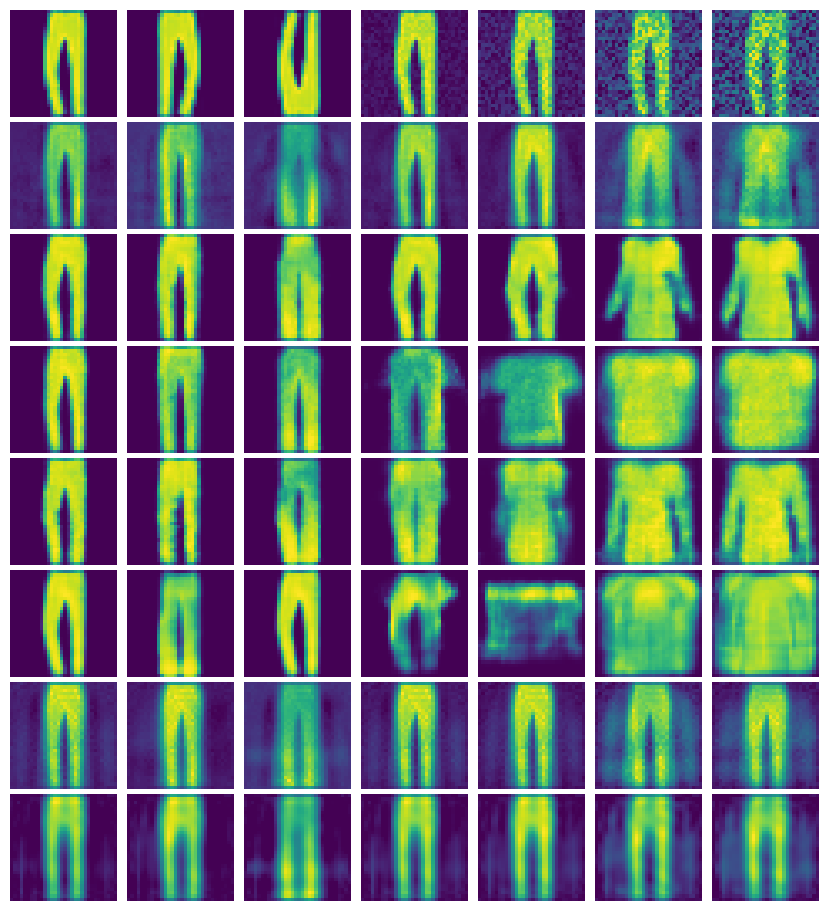}
                  \includegraphics[width=0.49\textwidth]{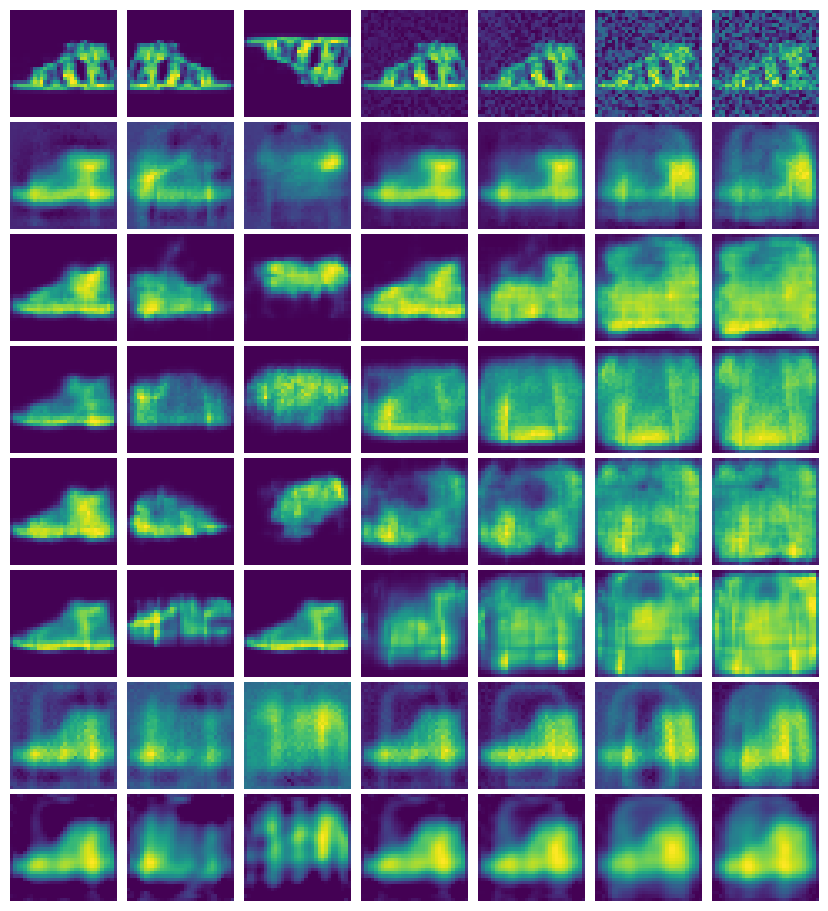}
         \caption{Two show cases of FashionMNIST reconstruction for latent dimension $m=10$. First row shows the input image with vertical, horizontal flips, and $0\%,10\%,20\%,50\%,70\% $ of Gaussian noise. Rows beneath show the results of (2) MLAP-AE, (3) CNN-AE, (4) MLP-VAE, (5) CNN-VAE, (6) ContraAE, (7) AE-REG, and (8) Hybrid AE-REG.}\label{fig:FM_showcase}
\end{figure*}

       Fig.~\ref{FM3},Fig.~\ref{FM4}  compare the reconstruction quality of the AEs for latent dimension $m=4$ with respect to SSIM. As before AE-REG and Hybrid AE-REG show worse performance than the other AEs for the clean images, but behave better than the other AEs for Gaussian noise perturbations. While SSIM correlates with the FashionMNIST classes stronger than the PSNR metric, the statistics here show that flipping the  FashionMNIST class (as in Fig.~\ref{fig:FM_showcase}) is prevented best by the regularised AEs.

       Fig.~\ref{FG1}-\ref{FGN} provide show cases of decoded latent-geodesics w.r.t. latent dimension $m=4$, connecting two AE-latent codes of the encoded test data that has been initially perturbed by
       $50\%$ Gaussian noise before encoding.  The geodesics are computed as shortest paths for
       an early Vietoris-Rips filtration \cite{moor2020topological} that contains the endpoints in one common connected component. More examples are given in the supplements.

       Note that while encoding is asked to preserve the topology but  not the geometry (might be non-conformal, distances and angles may change) shortest paths are not necessarily feasible geodesic approximates, which motivated our choice of the Vietoris-Rips filtration.

       The question we here pose is whether the geodesic connects the endpoints along the curved encoded data manifold $\Dc'=\varphi(\Dc)$ or includes forbidden short-cuts through $\Dc'$.
       Apart from CNN-AE and AE-REG all other geodesics contain latent codes of images belonging to another FashionMNIST-class, while for Hybrid AE-REG this happens just once. We interpret these appearances as forbidden short-cuts indicating that the latent representations
       violate the  requirement of topological preservation.

       AE-REG delivers a smoother transition between the endpoints than CNN-AE, suggesting that though the CNN-AE geodesic is shorter,  the regularised AEs preserve
       the topology with higher resolution.

       \begin{figure*}[t!]
                \includegraphics[width=1.0\textwidth]{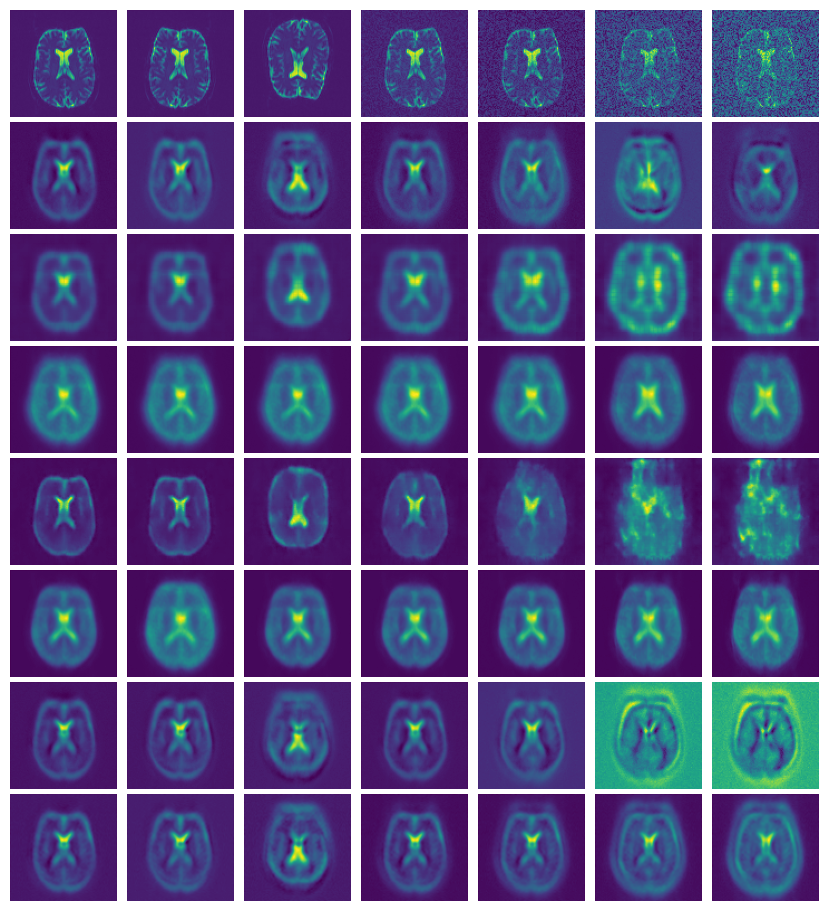}
                \caption{MRI show case. First row shows the input image with vertical, horizontal flips, and $0\%,10\%,20\%,50\%,70\% $ of Gaussian noise. Rows beneath show the results of (2) MLAP-AE, (3) CNN-AE, (4) MLP-VAE, (5) CNN-VAE, (6) ContraAE, (7) AE-REG, and (8) Hybrid AE-REG.}\label{fig:MRI_show}
       \end{figure*}

       \begin{figure}[t!]
                \centering
                \includegraphics[width=0.225\textwidth]{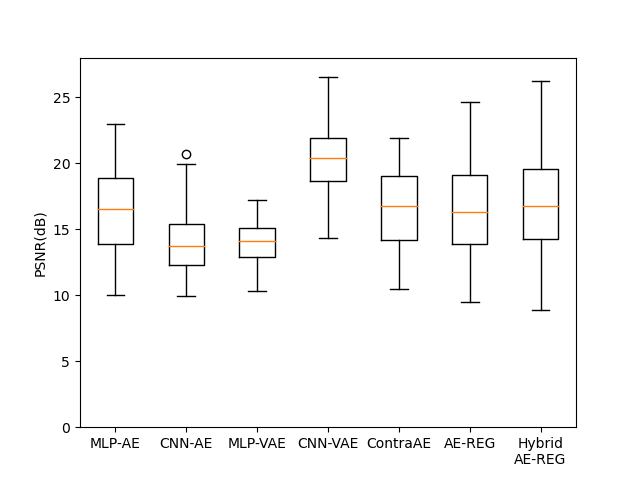}
                \includegraphics[width=0.225\textwidth]{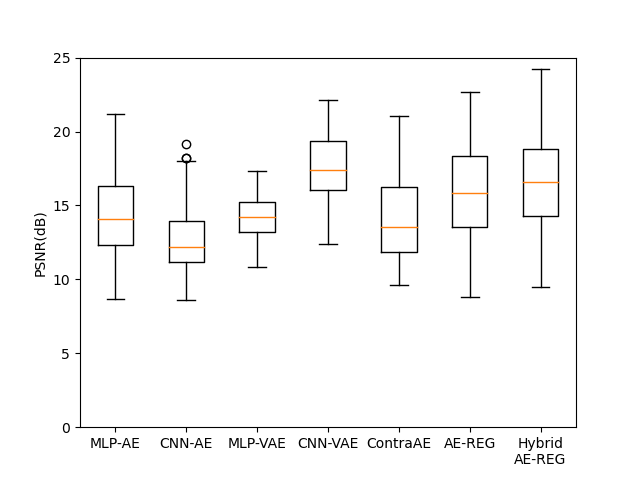}
               \caption{MRI reconstruction latent dimension $\dim =40$: without noise (left), $10\%$ Gaussian noise (right)}
               \label{MRI1}
       \end{figure}
       \begin{figure}[t!]
                \centering
                \includegraphics[width=0.225\textwidth]{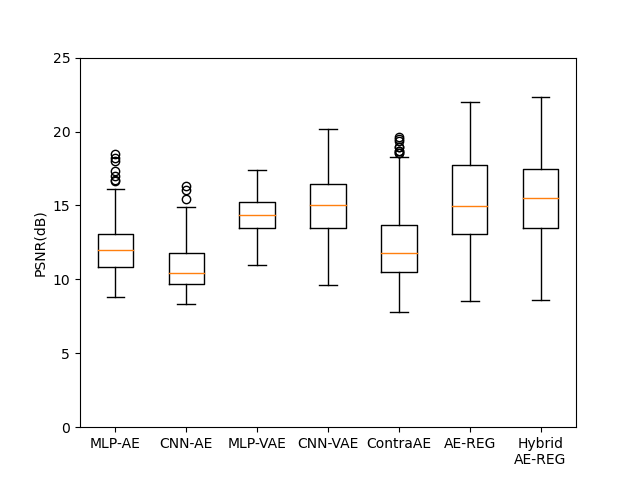}
                \includegraphics[width=0.225\textwidth]{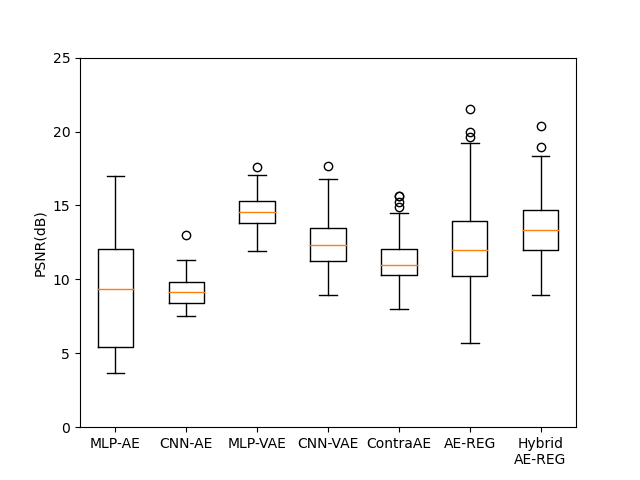}
               \caption{MRI reconstruction latent dimension $\dim =40$:  $20\%$ Gaussian noise (left), $50\%$ Gaussian noise (right)}
               \label{MRI2}
       \end{figure}

       \begin{figure}[!t]
         \centering
      \includegraphics[width=0.425\textwidth]{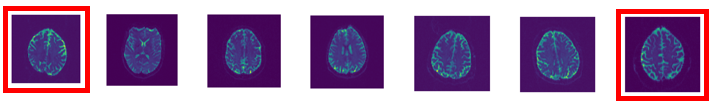}
        \caption{MRI goedesic in latent dimension $\dim =40$ of AE-REG without noise}
        \label{MG1}
       \end{figure}
       \begin{figure}[!t]
        \centering
       \includegraphics[width=0.5\textwidth]{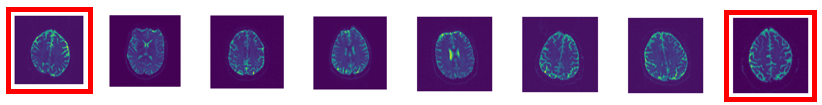}
       \caption{MRI goedesic in latent dimension $\dim =40$ of AE-REG with $10\%$ Gaussian noise}
       \end{figure}

       \begin{figure}[!t]
         \centering
         \includegraphics[width=0.325\textwidth]{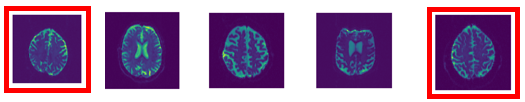}
         \caption{MRI goedesic in latent dimension $\dim =40$ of CNN-VAE without noise}
              \end{figure}
       \begin{figure}[!t]
         \centering
         \includegraphics[width=0.2\textwidth]{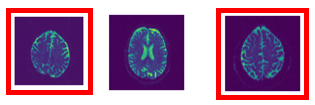}
           \caption{MRI goedesic in latent dimension $\dim =40$ of CNN-VAE with $10\%$ Gaussian noise}
              \end{figure}

       \begin{figure}[t!]
         \centering
         \includegraphics[width=0.2\textwidth]{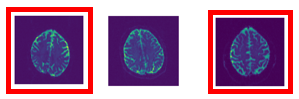}
   \caption{MRI goedesic in latent dimension $\dim =40$ of CNN-AE without noise}
              \end{figure}
       \begin{figure}[!t]
         \centering
         \includegraphics[width=0.5\textwidth]{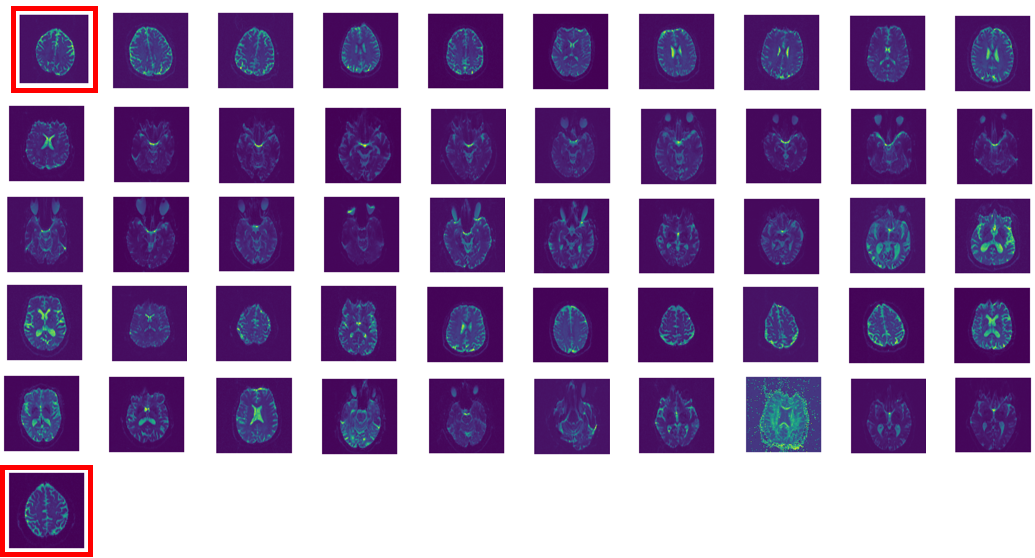}
           \caption{MRI goedesic in latent dimension $\dim =40$ of CNN-AE with $10\%$ Gaussian noise}
           \label{MGN}
        \end{figure}

       \subsection{Autoencoder compression for MRI brain scans}

For evaluating the potential impact of the hybridisation and regularisation technique  to high-dimensional \emph{"real-world-problems"} we conducted the following experiment.

       \begin{experiment}[MRI compression] We consider the MRI brain scans dataset from Open Access Series of Imaging Studies (OASIS) \cite{marcus2007open}. We extract two-dimensional slices from the three-dimensional MRI images, resulting  in $60.000$ images of resolution $64\times64$-pixels. We follow Experiment~\ref{exp:FM} by splitting the dataset into $40\%$ training images and complementary test images and compare the AE compression for latent dimension $m=40$. Results for latent dimension $m=10,15,20,40,60,70$ and $5\%,20\%,40\%,80\%$ training data are given in the supplements, as well as further details on the specifications.
       \end{experiment}

       We keep the architecture setup of the AEs, but increase the NN sizes to $5$ hidden layers each consisting of $1.000$ neurons. Reconstructions measured by PSNR are evaluated in Fig.~\ref{MRI1}, Fig.~\ref{MRI2}. Analogous results apper for SSIM, see the supplements.

       As in Experiment~\ref{exp:FM} we observe that AE-REG and Hybrid AE-REG perform compatible or slightly worse than the other AEs in the unperturbed scenario, but show their superiority over the other AEs for $10\%$ Gaussian noise, or $20\%$ for CNN-VAE. Especially Hybrid AE-REG maintains its reconstruction quality under noise perturbations up to $20\%$ (maintains stable for $50\%$). The performance increase compared to the un-regularised MLP-AE becomes evident and validates again that a  strong robustness is achieved due to the regularisation.

       A show case is given in Fig.~\ref{fig:MRI_show}. Apart from Hybrid AE-REG (row (8)) all AEs show artefacts when reconstructing perturbed images. CNN-VAE (row (4)) and AE-REG (row (7)) perform compatible and maintain stable up to $20\%$ Gaussian noise perturbation.

       In Fig.~\ref{MG1}-\ref{MGN} examples of geodesics are visualised, being computed  analogously as in Experiment~\ref{exp:FM} for the encoded images once without noise and once by adding $10\%$ Gaussian noise before encoding.  The AE-REG geodesic consists of similar slices, including one differing slice for $10\%$ Gaussian noise perturbation. CNN-VAE delivers a shorter path, however includes a strongly differing slice, which is kept for $10\%$ of Gaussian noise. CNN-AE provides a feasible geodesic in the unperturbed case, however becomes unstable in the perturbed case.

       We interpret the difference of the AE-REG to  CNN-VAE and CNN-AE as an indicator for delivering consistent latent representations on a higher resolution.
       While the CNN-AE and AE-REG geodesics indicate that one may trust the encoded latent representations, the CNN-AE encoding may not be suitable for reliable post-processing, such  as classification tasks. More show cases are given in the supplements, showing similar unstable behaviour of the other AEs.

       Summarising,
       the results validate once more regularisation and hybridisation to deliver reliable AEs that are capable for compressing real world datasets to low dimensional latent spaces by preserving their topology.
       How to extend the hybridisation technique to images or datasets of high resolution is one of the aspects we discuss in our concluding thoughts.

       \section{Conclusion}\label{sec:CON}
        \vskip6pt
       We delivered the mathematical theory for addressing encoding tasks of datasets being sampled from smooth data manifolds. Our insights condensed in an efficiently realisable regularisation constraint, resting on sampling the encoder Jacobian in Legendre nodes,  located in the latent space. We have proven the regularisation to guarantee a re-embedding of the data manifold under mild assumptions on the dataset. Consequently, the topological structure is preserved under regularised autoencoder compression.

       We want to stress that the regularisation is not limited to specific NN architectures, but already strongly impacts the performance for simple MLPs. Combinations with initially discussed vectorised AEs \cite{mitchell1980need,gordon1995evaluation} might extend and improve high-dimensional data analysis as in
       \cite{kobayashi2022self}.
       When combined with the proposed polynomial regression the hybridised AEs increase
       strongly in reconstruction quality. For addressing images of high resolution or multi-dimensional datasets, $\dim \geq 3$, as the un-sliced MRI brain scans,
       we propose to apply our recent extension of these regression methods \cite{REG_arxiv}.


       In summary, the regularised AEs performed by far better than the considered alternatives, especially when it comes to
       maintain the topological structure of the initial dataset.  The present computations of geodesics provides a tool for analysing the latent space geometry encoded by the regularised AEs and
       contributes towards explainability of reliable feature selections, as initially emphasised \cite{ronneberger2015u,galimov2022tandem,yakimovich2020mimicry,fisch2020image,andriasyan2021microscopy}.

       While structural preservation is substantial for consistent post-analysis, we believe that the proposed regularisation technique can deliver new reliable insights across disciplines and may even causes corrections or refinements of prior deduced correlations.

\bibliographystyle{IEEEtran}
\bibliography{Ref.bib}

\begin{thebibliography}{10}
\providecommand{\url}[1]{#1}
\csname url@samestyle\endcsname
\providecommand{\newblock}{\relax}
\providecommand{\bibinfo}[2]{#2}
\providecommand{\BIBentrySTDinterwordspacing}{\spaceskip=0pt\relax}
\providecommand{\BIBentryALTinterwordstretchfactor}{4}
\providecommand{\BIBentryALTinterwordspacing}{\spaceskip=\fontdimen2\font plus
\BIBentryALTinterwordstretchfactor\fontdimen3\font minus
  \fontdimen4\font\relax}
\providecommand{\BIBforeignlanguage}[2]{{%
\expandafter\ifx\csname l@#1\endcsname\relax
\typeout{** WARNING: IEEEtran.bst: No hyphenation pattern has been}%
\typeout{** loaded for the language `#1'. Using the pattern for}%
\typeout{** the default language instead.}%
\else
\language=\csname l@#1\endcsname
\fi
#2}}
\providecommand{\BIBdecl}{\relax}
\BIBdecl

\bibitem{pepperkok2006high}
R.~Pepperkok and J.~Ellenberg, ``High-throughput fluorescence microscopy for
  systems biology,'' \emph{Nature reviews Molecular cell biology}, vol.~7,
  no.~9, pp. 690--696, 2006.

\bibitem{perlman2004multidimensional}
Z.~E. Perlman, M.~D. Slack, Y.~Feng, T.~J. Mitchison, L.~F. Wu, and S.~J.
  Altschuler, ``Multidimensional drug profiling by automated microscopy,''
  \emph{Science}, vol. 306, no. 5699, pp. 1194--1198, 2004.

\bibitem{vogt2018machine}
N.~Vogt, ``Machine learning in neuroscience,'' \emph{Nature Methods}, vol.~15,
  no.~1, pp. 33--33, 2018.

\bibitem{carlson2018ghosts}
T.~Carlson, E.~Goddard, D.~M. Kaplan, C.~Klein, and J.~B. Ritchie, ``Ghosts in
  machine learning for cognitive neuroscience: Moving from data to theory,''
  \emph{NeuroImage}, vol. 180, pp. 88--100, 2018.

\bibitem{Zhang2020}
\BIBentryALTinterwordspacing
F.~Zhang, S.~{Cetin Karayumak}, N.~Hoffmann, Y.~Rathi, A.~J. Golby, and L.~J.
  O'Donnell, ``{Deep white matter analysis (DeepWMA): Fast and consistent
  tractography segmentation},'' \emph{Medical Image Analysis}, vol.~65, no.~1,
  p. 101761, oct 2020. [Online]. Available:
  \url{https://linkinghub.elsevier.com/retrieve/pii/S1361841520301250}
\BIBentrySTDinterwordspacing

\bibitem{karniadakis2021physics}
G.~E. Karniadakis, I.~G. Kevrekidis, L.~Lu, P.~Perdikaris, S.~Wang, and
  L.~Yang, ``Physics-informed machine learning,'' \emph{Nature Reviews
  Physics}, vol.~3, no.~6, pp. 422--440, 2021.

\bibitem{rodriguez2019identifying}
J.~F. Rodriguez-Nieva and M.~S. Scheurer, ``Identifying topological order
  through unsupervised machine learning,'' \emph{Nature Physics}, vol.~15,
  no.~8, pp. 790--795, 2019.

\bibitem{Willmann2021}
A.~Willmann, P.~Stiller, A.~Debus, A.~Irman, R.~Pausch, Y.-Y. Chang,
  M.~Bussmann, and N.~Hoffmann, ``Data-driven shadowgraph simulation of a 3d
  object,'' ser. Proceedings of Workshop Simulation with Deep Learning @ ICLR
  2021, 2021.

\bibitem{kobayashi2022self}
H.~Kobayashi, K.~C. Cheveralls, M.~D. Leonetti, and L.~A. Royer,
  ``Self-supervised deep learning encodes high-resolution features of protein
  subcellular localization,'' \emph{Nature methods}, vol.~19, no.~8, pp.
  995--1003, 2022.

\bibitem{chandrasekaran2021image}
S.~N. Chandrasekaran, H.~Ceulemans, J.~D. Boyd, and A.~E. Carpenter,
  ``Image-based profiling for drug discovery: due for a machine-learning
  upgrade?'' \emph{Nature Reviews Drug Discovery}, vol.~20, no.~2, pp.
  145--159, 2021.

\bibitem{anitei2014}
M.~Anitei, R.~Chenna, C.~Czupalla, M.~Esner, S.~Christ, S.~Lenhard, K.~Korn,
  F.~Meyenhofer, M.~Bickle, M.~Zerial \emph{et~al.}, ``A high-throughput sirna
  screen identifies genes that regulate mannose 6-phosphate receptor
  trafficking,'' \emph{Journal of cell science}, vol. 127, no.~23, pp.
  5079--5092, 2014.

\bibitem{nikitina2019}
K.~Nikitina, S.~Segeletz, M.~Kuhn, Y.~Kalaidzidis, and M.~Zerial, ``Basic
  phenotypes of endocytic system recognized by independent phenotypes analysis
  of a high-throughput genomic screen,'' in \emph{Proceedings of the 2019 3rd
  International Conference on Computational Biology and Bioinformatics}, 2019,
  pp. 69--75.

\bibitem{ronneberger2015u}
O.~Ronneberger, P.~Fischer, and T.~Brox, ``U-net: Convolutional networks for
  biomedical image segmentation,'' in \emph{Medical Image Computing and
  Computer-Assisted Intervention--MICCAI 2015: 18th International Conference,
  Munich, Germany, October 5-9, 2015, Proceedings, Part III 18}.\hskip 1em plus
  0.5em minus 0.4em\relax Springer, 2015, pp. 234--241.

\bibitem{galimov2022tandem}
E.~Galimov and A.~Yakimovich, ``A tandem segmentation-classification approach
  for the localization of morphological predictors of c. elegans lifespan and
  motility,'' \emph{Aging (Albany NY)}, vol.~14, no.~4, p. 1665, 2022.

\bibitem{yakimovich2020mimicry}
A.~Yakimovich, M.~Huttunen, J.~Samolej, B.~Clough, N.~Yoshida, S.~Mostowy,
  E.-M. Frickel, and J.~Mercer, ``Mimicry embedding facilitates advanced neural
  network training for image-based pathogen detection,'' \emph{Msphere},
  vol.~5, no.~5, pp. e00\,836--20, 2020.

\bibitem{fisch2020image}
D.~Fisch, A.~Yakimovich, B.~Clough, J.~Mercer, and E.-M. Frickel, ``Image-based
  quantitation of host cell--toxoplasma gondii interplay using hrman: A host
  response to microbe analysis pipeline,'' \emph{Toxoplasma gondii: Methods and
  protocols}, pp. 411--433, 2020.

\bibitem{andriasyan2021microscopy}
V.~Andriasyan, A.~Yakimovich, A.~Petkidis, F.~Georgi, R.~Witte, D.~Puntener,
  and U.~F. Greber, ``Microscopy deep learning predicts virus infections and
  reveals mechanics of lytic-infected cells,'' \emph{Iscience}, vol.~24, no.~6,
  p. 102543, 2021.

\bibitem{sanchez1979}
J.~S{\'a}nchez, K.~Mardia, J.~Kent, and J.~Bibby, \emph{Multivariate
  analysis}.\hskip 1em plus 0.5em minus 0.4em\relax Academic Press, London-New
  York-Toronto-Sydney-San Francisco, 1979.

\bibitem{dunteman1989basic}
G.~H. Dunteman, ``Basic concepts of principal components analysis,''
  \emph{Principal components analysis}, pp. 15--22, 1989.

\bibitem{krzanowski2000}
W.~Krzanowski, \emph{Principles of multivariate analysis}.\hskip 1em plus 0.5em
  minus 0.4em\relax OUP Oxford, 2000, vol.~23.

\bibitem{Rolinek2019}
\BIBentryALTinterwordspacing
M.~Rolinek, D.~Zietlow, and G.~Martius, ``{Variational Autoencoders Pursue PCA
  Directions (by Accident)},'' in \emph{2019 IEEE/CVF Conference on Computer
  Vision and Pattern Recognition (CVPR)}, vol. 2019-June.\hskip 1em plus 0.5em
  minus 0.4em\relax IEEE, jun 2019, pp. 12\,398--12\,407. [Online]. Available:
  \url{https://ieeexplore.ieee.org/document/8953837/}
\BIBentrySTDinterwordspacing

\bibitem{antun2021}
V.~Antun, N.~M. Gottschling, A.~C. Hansen, and B.~Adcock, ``Deep learning in
  scientific computing: Understanding the instability mystery,'' \emph{SIAM
  NEWS MARCH}, vol. 2021, 2021.

\bibitem{gottschling2020}
N.~M. Gottschling, V.~Antun, B.~Adcock, and A.~C. Hansen, ``The troublesome
  kernel: why deep learning for inverse problems is typically unstable,''
  \emph{arXiv preprint arXiv:2001.01258}, 2020.

\bibitem{antun2020}
V.~Antun, F.~Renna, C.~Poon, B.~Adcock, and A.~C. Hansen, ``On instabilities of
  deep learning in image reconstruction and the potential costs of {AI},''
  \emph{Proceedings of the National Academy of Sciences}, vol. 117, no.~48, pp.
  30\,088--30\,095, 2020.

\bibitem{NEURIPS2019_cd9508fd}
\BIBentryALTinterwordspacing
H.~Chen, H.~Zhang, S.~Si, Y.~Li, D.~Boning, and C.-J. Hsieh, ``Robustness
  verification of tree-based models,'' in \emph{Advances in Neural Information
  Processing Systems}, H.~Wallach, H.~Larochelle, A.~Beygelzimer,
  F.~d\textquotesingle Alch\'{e}-Buc, E.~Fox, and R.~Garnett, Eds.,
  vol.~32.\hskip 1em plus 0.5em minus 0.4em\relax Curran Associates, Inc.,
  2019. [Online]. Available:
  \url{https://proceedings.neurips.cc/paper/2019/file/cd9508fdaa5c1390e9cc329001cf1459-Paper.pdf}
\BIBentrySTDinterwordspacing

\bibitem{galhotra2017fairness}
S.~Galhotra, Y.~Brun, and A.~Meliou, ``Fairness testing: testing software for
  discrimination,'' in \emph{Proceedings of the 2017 11th Joint meeting on
  foundations of software engineering}, 2017, pp. 498--510.

\bibitem{mazzucato2021reduced}
D.~Mazzucato and C.~Urban, ``Reduced products of abstract domains for fairness
  certification of neural networks,'' in \emph{Static Analysis: 28th
  International Symposium, SAS 2021, Chicago, IL, USA, October 17--19, 2021,
  Proceedings 28}.\hskip 1em plus 0.5em minus 0.4em\relax Springer, 2021, pp.
  308--322.

\bibitem{lang1985}
S.~Lang, ``Differential manifolds,'' \emph{Springer}, 1985.

\bibitem{krantz2013}
S.~G. Krantz and H.~R. Parks, ``The implicit function theorem. modern
  birkh{\"a}user classics,'' \emph{History, theory, and applications, Reprint
  of the 2003 edition. Birkh{\"a}user/Springer, New York, pp. xiv}, vol. 163,
  2013.

\bibitem{stroud}
A.~Stroud, \emph{Approximate calculation of multiple integrals: Prentice-Hall
  series in automatic computation}.\hskip 1em plus 0.5em minus 0.4em\relax
  Prentice-Hall (Englewood Cliffs, NJ), 1971.

\bibitem{stroud2}
------, ``Secrest. d.(1966). {G}aussian quadrature formulas,'' 2011.

\bibitem{trefethen2017}
L.~N. Trefethen, ``Multivariate polynomial approximation in the hypercube,''
  \emph{Proceedings of the American Mathematical Society}, vol. 145, no.~11,
  pp. 4837--4844, 2017.

\bibitem{Trefethen2019}
------, \emph{Approximation theory and approximation practice}.\hskip 1em plus
  0.5em minus 0.4em\relax SIAM, 2019, vol. 164.

\bibitem{sobolev1997theory}
S.~L. Sobolev and V.~Vaskevich, \emph{The theory of cubature formulas}.\hskip
  1em plus 0.5em minus 0.4em\relax Springer Science \& Business Media, 1997,
  vol. 415.

\bibitem{REG_arxiv}
S.~K.~T. Veettil, Y.~Zheng, U.~H. Acosta, D.~Wicaksono, and M.~Hecht,
  ``Multivariate polynomial regression of euclidean degree extends the
  stability for fast approximations of trefethen functions,'' \emph{arXiv
  preprint arXiv:2212.11706}, 2022.

\bibitem{cardona2023learning}
J.-E. Suarez~Cardona, P.-A. Hofmann, and M.~Hecht, ``Learning partial
  differential equations by spectral approximates of general sobolev spaces,''
  \emph{arXiv preprint arXiv:2301.04887}, 2023.

\bibitem{esteban2022replacing}
J.-E. Suarez~Cardona and M.~Hecht, ``Replacing automatic differentiation by
  sobolev cubatures fastens physics informed neural nets and strengthens their
  approximation power,'' \emph{arXiv e-prints}, pp. arXiv--2211, 2022.

\bibitem{PIP1}
M.~Hecht, B.~L. Cheeseman, K.~B. Hoffmann, and I.~F. Sbalzarini, ``A
  quadratic-time algorithm for general multivariate polynomial interpolation,''
  \emph{arXiv preprint arXiv:1710.10846}, 2017.

\bibitem{PIP2}
M.~Hecht, K.~B. Hoffmann, B.~L. Cheeseman, and I.~F. Sbalzarini, ``Multivariate
  {N}ewton interpolation,'' \emph{arXiv preprint arXiv:1812.04256}, 2018.

\bibitem{MIP}
M.~Hecht, K.~Gonciarz, J.~Michelfeit, V.~Sivkin, and I.~F. Sbalzarini,
  ``Multivariate interpolation in unisolvent nodes--lifting the curse of
  dimensionality,'' \emph{arXiv preprint arXiv:2010.10824}, 2020.

\bibitem{hecht2018}
M.~Hecht and I.~F. Sbalzarini, ``Fast interpolation and {F}ourier transform in
  high-dimensional spaces,'' in \emph{Intelligent Computing. Proc. 2018 IEEE
  Computing Conf., Vol. 2,}, ser. Advances in Intelligent Systems and
  Computing, K.~Arai, S.~Kapoor, and R.~Bhatia, Eds., vol. 857.\hskip 1em plus
  0.5em minus 0.4em\relax London, UK: Springer Nature, 2018, pp. 53--75.

\bibitem{sup1}
K.~Sindhu~Meena and S.~Suriya, ``A survey on supervised and unsupervised
  learning techniques,'' in \emph{Proceedings of International Conference on
  Artificial Intelligence, Smart Grid and Smart City Applications}, L.~A.
  Kumar, L.~S. Jayashree, and R.~Manimegalai, Eds.\hskip 1em plus 0.5em minus
  0.4em\relax Cham: Springer International Publishing, 2020, pp. 627--644.

\bibitem{sup2}
\BIBentryALTinterwordspacing
G.~Chao, Y.~Luo, and W.~Ding, ``Recent advances in supervised dimension
  reduction: A survey,'' \emph{Machine Learning and Knowledge Extraction},
  vol.~1, no.~1, pp. 341--358, 2019. [Online]. Available:
  \url{https://www.mdpi.com/2504-4990/1/1/20}
\BIBentrySTDinterwordspacing

\bibitem{mitchell1980need}
T.~M. Mitchell, \emph{The need for biases in learning generalizations}.\hskip
  1em plus 0.5em minus 0.4em\relax Citeseer, 1980.

\bibitem{gordon1995evaluation}
D.~F. Gordon and M.~Desjardins, ``Evaluation and selection of biases in machine
  learning,'' \emph{Machine learning}, vol.~20, pp. 5--22, 1995.

\bibitem{wu2020vector}
H.~Wu and M.~Flierl, ``Vector quantization-based regularization for
  autoencoders,'' in \emph{Proceedings of the AAAI Conference on Artificial
  Intelligence}, vol.~34, no.~04, 2020, pp. 6380--6387.

\bibitem{van2017neural}
A.~Van Den~Oord, O.~Vinyals \emph{et~al.}, ``Neural discrete representation
  learning,'' \emph{Advances in neural information processing systems},
  vol.~30, 2017.

\bibitem{rifai2011higher}
S.~Rifai, G.~Mesnil, P.~Vincent, X.~Muller, Y.~Bengio, Y.~Dauphin, and
  X.~Glorot, ``Higher order contractive auto-encoder,'' in \emph{Joint European
  conference on machine learning and knowledge discovery in databases}.\hskip
  1em plus 0.5em minus 0.4em\relax Springer, 2011, pp. 645--660.

\bibitem{rifai2011contractive}
S.~Rifai, P.~Vincent, X.~Muller, X.~Glorot, and Y.~Bengio, ``Contractive
  auto-encoders: Explicit invariance during feature extraction,'' in
  \emph{ICML}, 2011.

\bibitem{kingma2013auto}
D.~P. Kingma and M.~Welling, ``Auto-encoding variational bayes,'' \emph{arXiv
  preprint arXiv:1312.6114}, 2013.

\bibitem{burgess2018understanding}
C.~P. Burgess, I.~Higgins, A.~Pal, L.~Matthey, N.~Watters, G.~Desjardins, and
  A.~Lerchner, ``Understanding disentangling in $\beta$-vae,'' \emph{arXiv
  preprint arXiv:1804.03599}, 2018.

\bibitem{Kumar2020}
A.~Kumar and B.~Poole, ``{On implicit regularization in $\beta$-VAEs},''
  \emph{37th International Conference on Machine Learning, ICML 2020}, vol.
  PartF168147-8, no.~Vi, pp. 5436--5446, 2020.

\bibitem{Rhodes2021}
T.~Rhodes and D.~Lee, ``Local disentanglement in variational auto-encoders
  using jacobian $ l\_1 $ regularization,'' \emph{Advances in Neural
  Information Processing Systems}, vol.~34, pp. 22\,708--22\,719, 2021.

\bibitem{gilbert2017}
A.~C. Gilbert, Y.~Zhang, K.~Lee, Y.~Zhang, and H.~Lee, ``Towards understanding
  the invertibility of convolutional neural networks,'' in \emph{Proceedings of
  the 26th International Joint Conference on Artificial Intelligence}, 2017,
  pp. 1703--1710.

\bibitem{martin2009}
M.~Anthony and P.~L. Bartlett, \emph{Neural network learning: Theoretical
  foundations}.\hskip 1em plus 0.5em minus 0.4em\relax Cambridge University
  press, 2009.

\bibitem{goodfellow2016}
I.~Goodfellow, Y.~Bengio, and A.~Courville, \emph{Deep learning}.\hskip 1em
  plus 0.5em minus 0.4em\relax MIT press, 2016.

\bibitem{Adams2003}
R.~A. Adams and J.~J. Fournier, \emph{Sobolev spaces}.\hskip 1em plus 0.5em
  minus 0.4em\relax Academic press, 2003, vol. 140.

\bibitem{brezis2011}
H.~Brezis, \emph{Functional analysis, Sobolev spaces and partial differential
  equations}.\hskip 1em plus 0.5em minus 0.4em\relax Springer, 2011, vol.~2.

\bibitem{pedersen2018}
M.~Pedersen, \emph{Functional analysis in applied mathematics and
  engineering}.\hskip 1em plus 0.5em minus 0.4em\relax CRC press, 2018.

\bibitem{gautschi2011}
W.~Gautschi, \emph{Numerical analysis}.\hskip 1em plus 0.5em minus 0.4em\relax
  Springer Science \& Business Media, 2011.

\bibitem{chen1999}
W.~Chen, S.-s. Chern, and K.~S. Lam, \emph{Lectures on differential
  geometry}.\hskip 1em plus 0.5em minus 0.4em\relax World Scientific Publishing
  Company, 1999, vol.~1.

\bibitem{taubes2011}
C.~H. Taubes, \emph{Differential geometry: Bundles, connections, metrics and
  curvature}.\hskip 1em plus 0.5em minus 0.4em\relax OUP Oxford, 2011, vol.~23.

\bibitem{do2016}
M.~P. Do~Carmo, \emph{Differential geometry of curves and surfaces: revised and
  updated second edition}.\hskip 1em plus 0.5em minus 0.4em\relax Courier Dover
  Publications, 2016.

\bibitem{weier1}
K.~Weierstrass, ``{\"U}ber die analytische {D}arstellbarkeit sogenannter
  willk{\"u}rlicher {F}unktionen einer reellen {V}er{\"a}nderlichen,''
  \emph{Sitzungsberichte der K{\"o}niglich Preu{\ss}ischen Akademie der
  Wissenschaften zu Berlin}, vol.~2, pp. 633--639, 1885.

\bibitem{weier2}
L.~De~Branges, ``The {Stone-Weierstrass Theorem},'' \emph{Proceedings of the
  American Mathematical Society}, vol.~10, no.~5, pp. 822--824, 1959.

\bibitem{baydin2018}
A.~G. Baydin, B.~A. Pearlmutter, A.~A. Radul, and J.~M. Siskind, ``Automatic
  differentiation in machine learning: a survey,'' \emph{Journal of machine
  learning research}, vol.~18, 2018.

\bibitem{bezout1779}
E.~B{\'e}zout, \emph{Th{\'e}orie g{\'e}n{\'e}rale des {\'e}quations
  alg{\'e}briques par M. B{\'e}zout..}\hskip 1em plus 0.5em minus 0.4em\relax
  de l'imprimerie de Ph.-D. Pierres, rue S. Jacques, 1779.

\bibitem{fulton1974}
W.~Fulton, ``Algebraic curves (mathematics lecture note series),'' 1974.

\bibitem{kingma2014}
D.~P. Kingma and J.~Ba, ``Adam: A method for stochastic optimization,''
  \emph{arXiv preprint arXiv:1412.6980}, 2014.

\bibitem{xiao2017/online}
H.~Xiao, K.~Rasul, and R.~Vollgraf. (2017) Fashion-mnist: a novel image dataset
  for benchmarking machine learning algorithms.

\bibitem{moor2020topological}
M.~Moor, M.~Horn, B.~Rieck, and K.~Borgwardt, ``Topological autoencoders,'' in
  \emph{International conference on machine learning}.\hskip 1em plus 0.5em
  minus 0.4em\relax PMLR, 2020, pp. 7045--7054.

\bibitem{marcus2007open}
D.~S. Marcus, T.~H. Wang, J.~Parker, J.~G. Csernansky, J.~C. Morris, and R.~L.
  Buckner, ``Open access series of imaging studies (oasis): cross-sectional mri
  data in young, middle aged, nondemented, and demented older adults,''
  \emph{Journal of cognitive neuroscience}, vol.~19, no.~9, pp. 1498--1507,
  2007.

\end{thebibliography}

\vfill

\end{document}